\documentclass{llncs}
\usepackage{llncsdoc}
\usepackage{amsmath}
\usepackage{amsfonts}
\usepackage[dvipsnames]{xcolor}
\usepackage{graphicx}
\usepackage{marvosym}
\usepackage{multirow}
\usepackage[hidelinks]{hyperref}
\usepackage{breakcites}
\usepackage{endnotes}
\usepackage[font=small,format=plain,labelfont=bf,up,textfont=small,up,justification=justified,singlelinecheck=false]{caption}
\newcommand\Tstrut{\rule{0pt}{2.6ex}}
\title{Formal Ways for Measuring Relations between Concepts in Conceptual Spaces\thanks{This article is a revised and considerably extended version of a conference paper presented at SGAI 2017 \cite{Bechberger2017SGAI}.}}
\author{Lucas Bechberger \Letter(0000-0002-1962-1777) \and Kai-Uwe K\"uhnberger}
\institute{Institute of Cognitive Science, Osnabr\"uck University, Osnabr\"uck, Germany \email{lucas.bechberger@uni-osnabrueck.de}, \email{kai-uwe.kuehnberger@uni-osnabrueck.de}}

\begin{document}
\maketitle

%================================================================================================================================================================%

\begin{abstract}
The highly influential framework of conceptual spaces provides a geometric way of representing knowledge. Instances are represented by points in a high-dimensional space and concepts are represented by regions in this space. In this article, we extend our recent mathematical formalization of this framework by providing quantitative mathematical definitions for measuring relations between concepts: We develop formal ways for computing concept size, subsethood, implication, similarity, and betweenness. This considerably increases the representational capabilities of our formalization and makes it the most thorough and comprehensive formalization of conceptual spaces developed so far.
\end{abstract}

%================================================================================================================================================================%
\section{Introduction}
\label{Intro}

One common criticism of symbolic AI approaches is that the symbols they operate on do not contain any meaning: For the system, they are just arbitrary tokens that can be manipulated in some way. This lack of inherent meaning in abstract symbols is called the “symbol grounding problem” \cite{Harnad1990}. One approach towards solving this problem is to devise a grounding mechanism that connects  abstract symbols to the real world, i.e., to perception and action.

The cognitive framework of conceptual spaces \cite{Gardenfors2000,Gardenfors2014} attempts to bridge this gap between symbolic and subsymbolic AI by proposing an intermediate conceptual layer based on geometric representations.
A conceptual space is a similarity space spanned by a number of quality dimensions that are based on perception and/or subsymbolic processing. Regions in this space correspond to concepts and can be referred to as abstract symbols.

The framework of conceptual spaces has been highly influential in the last 15 years within cognitive science \cite{Douven2011,Fiorini2013,Lieto2017}. It has also sparked considerable research in various subfields of artificial intelligence, ranging from robotics and computer vision \cite{Chella2003} over the semantic web \cite{Adams2009a} to plausible reasoning \cite{Derrac2015}.

One important aspect of conceptual representations is however often ignored by these research efforts: Typically, the different features of a concept are correlated with each other. For instance, there is an obvious correlation between the color and the taste of an apple: Red apples tend to be sweet and green apples tend to be sour. Recently, we have proposed a formalization of the conceptual spaces framework that is capable of representing such correlations in a geometric way \cite{Bechberger2017KI}. Our formalization not only contains a parametric definition of concepts, but also different operations to create new concepts from old ones (namely: intersection, unification, and projection). 

In this article, we provide mathematical definitions for the notions of concept size, subsethood, implication, similarity, and betweenness. This considerably increases the representational power of our formalization by introducing measurable ways of describing relations between concepts.

The remainder of this article is structured as follows:
Section \ref{CS} introduces the general framework of conceptual spaces along with our recent formalization. In Section \ref{Extension}, we extend this formalization with additional operations and in Section \ref{Example} we provide an illustrative example. Section \ref{RelatedWork} contains a summary of related work and Section \ref{Conclusion} concludes the paper.

%================================================================================================================================================================%
\section{Conceptual Spaces}
\label{CS}

This section presents the cognitive framework of conceptual spaces as described by \cite{Gardenfors2000} and as formalized by \cite{Bechberger2017KI}.

%--------------------------------------------------------------------------------------------------------------------------------------------------------------------------------------------------------------------------------------------------------------------------------------------------------------%
\subsection{Dimensions, Domains, and Distance}
\label{CS:DimensionsDomainsDistance}

%dimensions
A conceptual space is a similarity space spanned by a set $D$ of so-called ``quality dimensions''. Each of these dimensions $d \in D$ represents an interpretable way in which two stimuli can be judged to be similar or different. Examples for quality dimensions include temperature, weight, time, pitch, and hue. The distance between two points $x$ and $y$ with respect to a dimension $d$ is denoted as $|x_d - y_d|$.

%domains
A domain $\delta \subseteq D$ is a set of dimensions that inherently belong together. Different perceptual modalities (like color, shape, or taste) are represented by different domains. The color domain for instance consists of the three dimensions hue, saturation, and brightness. Distance within a domain $\delta$ is measured by the weighted Euclidean metric $d_E$.  

The overall conceptual space $CS$ is defined as the product space of all dimensions. Distance within the overall conceptual space is measured by the weighted Manhattan metric $d_M$ of the intra-domain distances. This is supported by both psychological evidence \cite{Attneave1950,Shepard1964,Johannesson2001} and mathematical considerations \cite{Aggarwal2001}. Let $\Delta$ be the set of all domains in $CS$. The combined distance $d_C^{\Delta}$ within $CS$ is defined as follows:
$$
d_C^{\Delta}(x,y,W) = \sum_{\delta \in \Delta}w_{\delta} \cdot \sqrt{\sum_{d \in \delta} w_{d} \cdot |x_{d} - y_{d}|^2}
$$
The parameter $W = \langle W_{\Delta},\{W_{\delta}\}_{\delta \in \Delta}\rangle$ contains two parts: $W_{\Delta}$ is the set of positive domain weights $w_{\delta}$ with $\textstyle\sum_{\delta \in \Delta} w_{\delta} = |\Delta|$. Moreover, $W$ contains for each domain $\delta \in \Delta$ a set $W_{\delta}$ of dimension weights $w_{d}$ with $\textstyle\sum_{d \in \delta} w_{d} = 1$.\\

%similarity as distance-based with exponential decay
The similarity of two points in a conceptual space is inversely related to their distance. This can be written as follows :
$$Sim(x,y) = e^{-c \cdot d(x,y)}\quad \text{with a constant}\; c >0 \; \text{and a given metric}\; d$$

%betweenness and convexity (formal)
Betweenness is a logical predicate $B(x,y,z)$ that is true if and only if $y$ is considered to be between $x$ and $z$. It can be defined based on a given metric $d$: 
$$B_d(x,y,z) :\iff d(x,y) + d(y,z) = d(x,z)$$

The betweenness relation based on $d_E$ results in the line segment connecting the points $x$ and $z$, whereas the betweenness relation based on $d_M$ results in an axis-parallel cuboid between the points $x$ and $z$.
One can define convexity and star-shapedness based on the notion of betweenness:

\begin{definition}
\label{def:Convexity}
(Convexity)\\
A set $C \subseteq CS$ is \emph{convex} under a metric $d \;:\iff$

\hspace{1cm}$\forall {x \in C, z \in C, y \in CS}: \left(B_d(x,y,z) \rightarrow y \in C\right)$
\end{definition}

\begin{definition}
\label{def:StarShapedSet}
(Star-shapedness)\\
A set $S \subseteq CS$ is \emph{star-shaped} under a metric $d$ with respect to a set $P \subseteq S \;:\iff$ 

\hspace{1cm}$\forall {p \in P, z \in S, y \in CS}: \left(B_d(p,y,z) \rightarrow y \in S\right)$
\end{definition}

%--------------------------------------------------------------------------------------------------------------------------------------------------------------------------------------------------------------------------------------------------------------------------------------------------------------%
\subsection{Properties and Concepts}
\label{CS:PropertiesConcepts}

\cite{Gardenfors2000} observes that properties like ``red'', ``round'', and ``sweet'' can be defined on individual domains (e.g., color, shape, taste), whereas full-fleshed concepts like ``apple'' or ``dog'' involve multiple domains.
Each domain involved in representing a concept has a certain importance, which is reflected by so-called ``salience weights''. Another important aspect of concepts are the correlations between the different domains, which are important for both learning \cite{Billman1996} and reasoning \cite[Ch 8]{Murphy2002}.

Based on the principle of cognitive economy, G\"{a}rdenfors argues that both properties and concepts should be represented as convex sets. However, as we demonstrated in \cite{Bechberger2017KI}, one cannot geometrically encode correlations between domains when using convex sets:
The left part of Figure \ref{fig:ConvexityProblem} shows two domains (age and height) with an intuitive sketch for the concepts of child and adult. As domains are combined with the Manhattan metric, a convex set corresponds in this case to an axis-parallel cuboid. One can easily see that this convex representation (middle part of Figure \ref{fig:ConvexityProblem}) is not satisfactory, because the correlation of the two domains is not encoded. We therefore proposed in \cite{Bechberger2017KI} to relax the convexity criterion and to use star-shaped sets, which is illustrated in the right part of Figure \ref{fig:ConvexityProblem}. This enables a geometric representation of correlations while still being only a minimal departure from the original framework.\\
\begin{figure}[tp]
\centering
\includegraphics[width=\columnwidth]{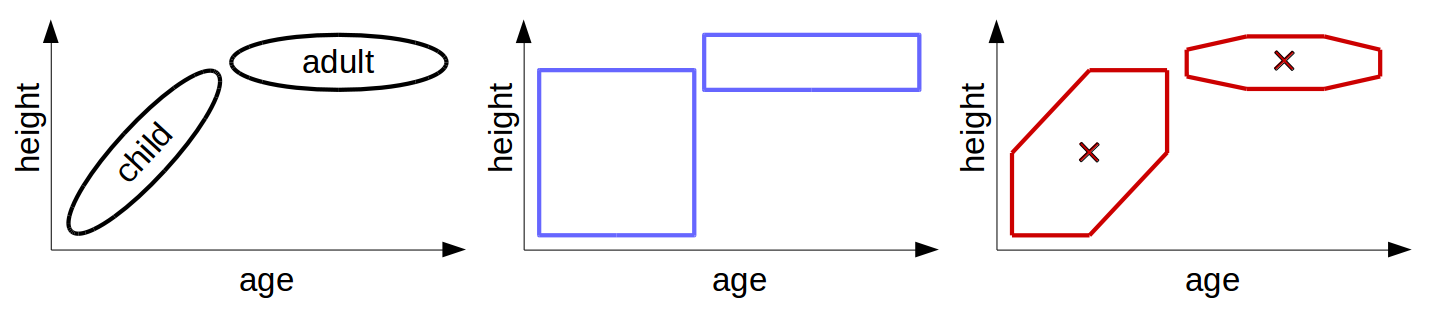}
\caption{Left: Intuitive way to define regions for the concepts of ``adult'' and ``child''. Middle: Representation by using convex sets. Right: Representation by using star-shaped sets with central points marked by crosses.}
\label{fig:ConvexityProblem}
\end{figure}

We have based our formalization on axis-parallel cuboids that can be described by a triple $\langle \Delta_C, p^-, p^+ \rangle$ consisting of a set of domains $\Delta_C$ on which this cuboid $C$ is defined and two points $p^-$ and $p^+$, such that 
$$x \in C \iff \forall{\delta \in \Delta_C}:\forall{d \in \delta}: p_d^- \leq x_d \leq p_d^+$$
These cuboids are convex under $d_C^{\Delta}$. It is also easy to see that any union of convex sets that have a non-empty intersection is star-shaped \cite{Smith1968}. We define the core of a concept as follows:

\begin{definition}
\label{def:SSSS}
(Simple star-shaped set)\\
A \emph{simple star-shaped set} $S$ is described as a tuple $\langle\Delta_S,\{C_1,\dots,C_m\}\rangle$. $\Delta_S \subseteq \Delta$ is a set of domains on which the cuboids $\{C_1,\dots,C_m\}$ (and thus also $S$) are defined. Moreover, it is required that the central region $P :=\textstyle\bigcap_{i = 1}^m C_i \neq \emptyset$. Then the simple star-shaped set $S$ is defined as 
$$S := \bigcup_{i=1}^m C_i$$
\end{definition}

In order to represent imprecise concept boundaries, we use fuzzy sets \cite{Zadeh1965,Bvelohlavek2011}. A fuzzy set is characterized by its membership function $\mu: CS \rightarrow [0,1]$ that assigns a degree of membership to each point in the conceptual space. The membership of a point to a fuzzy concept is based on its maximal similarity to any of the points in the concept's core:

\begin{definition}
\label{def:FSSSS}
(Fuzzy simple star-shaped set)\\
A \emph{fuzzy simple star-shaped set} $\widetilde{S}$ is described by a quadruple $\langle S,\mu_0,c,W\rangle$ where
$S = \langle\Delta_S,\{C_1,\dots,C_m\}\rangle$ is a non-empty simple star-shaped set. The parameter $\mu_0 \in (0,1]$ controls the highest possible membership to $\widetilde{S}$ and is usually set to 1. The sensitivity parameter $c > 0$ controls the rate of the exponential decay in the similarity function. Finally, $W = \langle W_{\Delta_S},\{W_{\delta}\}_{\delta \in \Delta_S}\rangle$ contains positive weights for all domains in $\Delta_S$ and all dimensions within these domains, reflecting their respective importance. We require that $\textstyle\sum_{\delta \in \Delta_S} w_{\delta} = |\Delta_S|$ and that $\forall {\delta \in \Delta_S}:\textstyle\sum_{d \in \delta} w_{d} = 1$.
The membership function of $\widetilde{S}$ is then defined as follows:
$$\mu_{\widetilde{S}}(x) = \mu_0 \cdot \max_{y \in S}\left(e^{-c \cdot d_C^{\Delta}(x,y,W)}\right)$$
\end{definition}

The sensitivity parameter $c$ controls the overall degree of fuzziness of $\widetilde{S}$ by determining how fast the membership drops to zero. The weights $W$ represent not only the relative importance of the respective domain or dimension for the represented concept, but they also influence the relative fuzziness with respect to this domain or dimension.
Note that if $|\Delta_S| = 1$, then $\widetilde{S}$ represents a property, and if $|\Delta_S| > 1$, then $\widetilde{S}$ represents a concept.
Figure \ref{fig:FSSSS} illustrates this definition (the $x$ and $y$ axes are assumed to belong to different domains and are combined with $d_M$ using equal weights).
\begin{figure}[tp]
\centering
\includegraphics[width = \columnwidth]{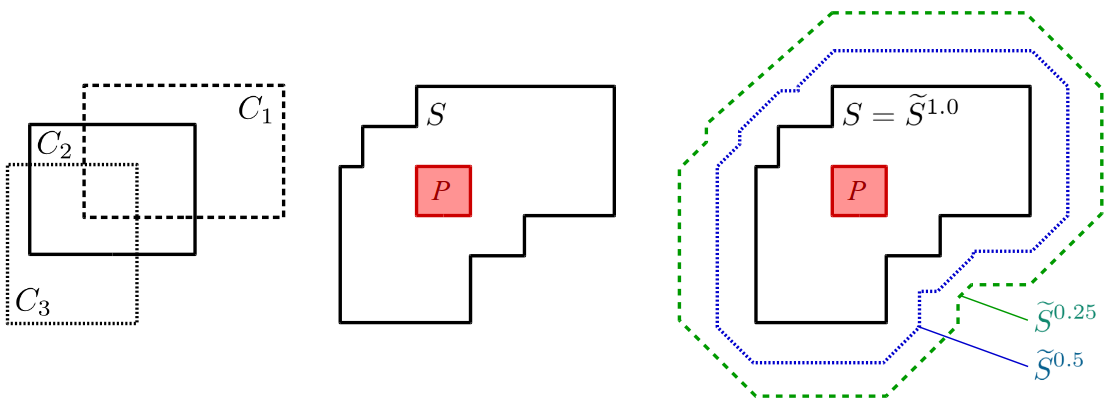} 
\caption{Left: Three cuboids $C_1, C_2, C_3$ with nonempty intersection. Middle: Resulting simple star-shaped set $S$ based on these cuboids. Right: Fuzzy simple star-shaped set $\tilde{S}$ based on $S$ with three $\alpha$-cuts for $\alpha \in \{1.0,0.5,0.25\}$.}
\label{fig:FSSSS}
\end{figure}

In our previous work \cite{Bechberger2017KI}, we have also provided a number of operations, which can be used to create new concepts from old ones, namely intersection, unification, and projection: The intersection of two concepts can be interpreted as the logical ``and'' -- e.g., intersecting the property ``green'' with the concept ``banana'' results in the set of all objects that are both green and bananas. The unification of two concepts can be used to construct more abstract categories (e.g., defining ``fruit'' as the unification of ``apple'', ``banana'', ``coconut'', etc.). Projecting a concept onto a subspace corresponds to focusing on certain domains while completely ignoring others.

%================================================================================================================================================================%
\section{Defining Additional Operations}
\label{Extension}

\subsection{Concept Size}
\label{Extension:Hypervolume}
The size of a concept gives an intuition about its specificity: Large concepts are more general and small concepts are more specific. This is one obvious aspect in which one can compare two concepts to each other.

One can use a measure $M$ to describe the size of a fuzzy set. It can be defined in our context as follows (cf. \cite{Bouchon-Meunier1996}):
\begin{definition}
A measure $M$ on a conceptual space $CS$ is a function $M: \mathcal{F}(CS) \rightarrow \mathbb{R}^+_0$ with $M(\emptyset) = 0$ and $\widetilde{A} \subseteq \widetilde{B} \Rightarrow M(\widetilde{A}) \leq M(\widetilde{B})$, where $\mathcal{F}(CS)$ is the fuzzy power set of $CS$ and where $\widetilde{A} \subseteq \widetilde{B} :\iff \forall {x \in CS}: \mu_{\widetilde{A}}(x) \leq \mu_{\widetilde{B}}(x)$.
\end{definition}

A common measure for fuzzy sets is the integral over the set's membership function, which is equivalent to the Lebesgue integral over the fuzzy set's $\alpha$-cuts:
\begin{equation}
M(\widetilde{A}) := \int_{CS} \mu_{\widetilde{A}}(x)\; dx = \int_{0}^1 V(\widetilde{A}^{\alpha})\; d\alpha
\label{eqn:integral}
\end{equation}
We use $V(\widetilde{A}^{\alpha})$ to denote the volume of a fuzzy set's $\alpha$-cut $\widetilde{A}^{\alpha} = \{ x \in CS \;|\; \mu_{\widetilde{A}}(x) \geq \alpha\}$.
We define for each cuboid $C_i \in S$ its fuzzified version $\widetilde{C}_i$ (cf. Definition \ref{def:FSSSS}):
$$\mu_{\widetilde{C}_i}(x) = \mu_0 \cdot \max_{y \in C_i}\left(e^{-c \cdot d_C^{\Delta}(x,y,W)}\right)$$
It is obvious that $\mu_{\widetilde{S}}(x) = \max_{C_i \in S} \mu_{\widetilde{C}_i}(x)$. It is also clear that the intersection of two fuzzified cuboids is again a fuzzified cuboid. Finally, one can easily see that we can use the inclusion-exclusion formula (cf. e.g., \cite{Bogart1989}) to compute the overall measure of $\widetilde{S}$ based on the measure of its fuzzified cuboids:
\begin{equation}
M(\widetilde{S}) = \sum_{l=1}^m \left((-1)^{l+1} \cdot \sum_{\substack{\{i_1,\dots,i_l\}\\\subseteq\{1,\dots,m\}}}M\left(\bigcap_{i \in \{i_1,\dots,i_l\}} \widetilde{C}_i\right)\right)
\label{eqn:inclusionExclusion}
\end{equation}
The outer sum iterates over the number of cuboids under consideration (with $m$ being the total number of cuboids in S) and the inner sum iterates over all sets of exactly $l$ cuboids. The overall formula generalizes the observation that $|\widetilde{C}_1 \cup \widetilde{C}_2| = |\widetilde{C}_1| + |\widetilde{C}_2| - |\widetilde{C}_1 \cap \widetilde{C}_2|$ from two to $m$ sets.

In order to derive $M(\widetilde{S})$, we first describe how to compute $V(\widetilde{C}^{\alpha})$, i.e., the size of a fuzzified cuboid's $\alpha$-cut. Using Equation \ref{eqn:integral}, we can then derive $M(\widetilde{C})$, which we can in turn insert into Equation \ref{eqn:inclusionExclusion} to compute the overall size of $\widetilde{S}$.\\

Figure \ref{fig:2DAlphaCut} illustrates the $\alpha$-cut of a fuzzified two-dimensional cuboid both under $d_E$ (left) and under $d_M$ (right). Because the membership function is defined based on an exponential decay, one can interpret each $\alpha$-cut as an $\epsilon$-neighborhood of the original cuboid C, where $\epsilon$ depends on $\alpha$: 
$$
x \in {\widetilde{C}}^{\alpha} \iff \mu_0 \cdot \max_{y \in C}\left(e^{-c \cdot d_C^{\Delta}(x,y,W)}\right) \geq \alpha \iff  \min_{y \in C} d_C^{\Delta}(x,y,W) \leq -\frac{1}{c} \cdot \ln\left(\frac{\alpha}{\mu_0}\right)
$$
\begin{figure}[tp]
\centering
\includegraphics[width = 0.75\textwidth]{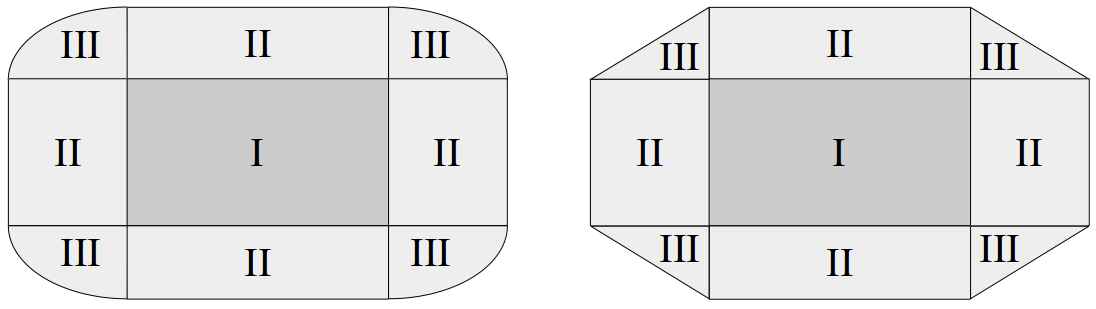}
\caption{$\alpha$-cut of a fuzzified cuboid under $d_E$ (left) and $d_M	$ (right), respectively.}
\label{fig:2DAlphaCut}
\end{figure}

$V(\widetilde{C}^{\alpha})$ can be described as a sum of different components. Let us use the shorthand notation $b_d := p_d^+ - p_d^-$.
Looking at Figure \ref{fig:2DAlphaCut}, one can see that all components of $V(\widetilde{C}^{\alpha})$ can be described by ellipses: Component I is a zero-dimensional ellipse (i.e., a point) that was extruded in two dimensions with extrusion lengths of $b_1$ and $b_2$, respectively. Component II consists of two one-dimensional ellipses (i.e., line segments) that were extruded in one dimension, and component III is a two-dimensional ellipse. Note that ellipses under $d_M$ have the form of streched diamonds.

Let us denote by $\Delta_{\{d_1,\dots,d_i\}}$ the domain structure obtained by eliminating from $\Delta$ all dimensions $d \in D\setminus\{d_1, \dots, d_i\}$. Moreover, let $V(r, \Delta, W)$ be the hypervolume of a hyperball under $d_C^\Delta(\cdot,\cdot, W)$ with radius $r$. In this case, a hyperball is the set of all points with a distance of at most $r$ (measured by $d_C^\Delta(\cdot,\cdot, W)$) to a central point. Note that the weights $W$ will cause this ball to have the form of an ellipse. For instance, in Figure \ref{fig:2DAlphaCut}, we assume that $w_{d_1} < w_{d_2}$ which means that we allow larger differences with respect to $d_1$ than with respect to $d_2$. This causes the hyperballs to be streched in the $d_1$ dimension, thus obtaining the shape of an ellipse.
We can in general describe $V(\widetilde{C}^{\alpha})$ as follows:
$$V(\widetilde{C}^{\alpha}) = \sum_{i=0}^n \left( \sum_{\substack{\{d_1,\dots,d_i\}\\ \subseteq D}} \left( \prod_{\substack{d \in\\ D\setminus\{d_1,\dots,d_i\}}} b_d \right) \cdot V\left( -\frac{1}{c} \cdot \ln\left(\frac{\alpha}{\mu_0}\right),\Delta_{\{d_1, \dots, d_i\}},W \right)\right)$$
The outer sum of this formula runs over the number of dimensions with respect to which a given point $x \in \widetilde{C}^{\alpha}$ lies outside of $C$. We then sum over all combinations $\{d_1,\dots,d_i\}$ of dimensions for which this could be the case, compute the volume $V(\cdot,\cdot,\cdot)$ of the $i$-dimensional hyperball under these dimensions and extrude this intermediate result in all remaining dimensions by multiplying with $\prod_{d \in D\setminus\{d_1,\dots,d_i\}} b_d$.

Let us illustrate this formula for the $\alpha$-cuts shown in Figure \ref{fig:2DAlphaCut}: For $i = 0$, we can only select the empty set for the inner sum, so we end up with $b_1 \cdot b_2$, which is the size of the original cuboid (i.e., component I). For $i = 1$, we can either pick $\{d_1\}$ or $\{d_2\}$ in the inner sum. For $\{d_1\}$, we compute the size of the left and right part of component II by multiplying $V\left( -\frac{1}{c} \cdot \ln\left(\frac{\alpha}{\mu_0}\right),\Delta_{\{d_1\}},W \right)$ (i.e., their combined width) with $b_2$ (i.e., their height). For $\{d_2\}$, we analogously compute the size of the upper and the lower part of component II. Finally, for $i = 2$, we can only pick $\{d_1, d_2\}$ in the inner sum, leaving us with $V\left( -\frac{1}{c} \cdot \ln\left(\frac{\alpha}{\mu_0}\right),\Delta ,W \right)$, which is the size of component III.
One can easily see that the formula for $V(\widetilde{C}^{\alpha})$ also generalizes to higher dimensions.

\begin{proposition}
\label{proposition:hyperball}
$V(r,\Delta, W) = \frac{1}{\prod_{\delta \in \Delta} w_{\delta} \cdot \prod_{d \in \delta} \sqrt{w_d}} \cdot \frac{r^n}{n!} \cdot \prod_{\delta \in \Delta} \left(|\delta|! \cdot \frac{\pi^{\frac{|\delta|}{2}}}{\Gamma\left(\frac{|\delta|}{2}+1\right)}\right)$
\end{proposition}
\begin{proof}
See Appendix A. \qed
\end{proof}

Defining $\delta(d)$ as the unique $\delta \in \Delta$ with $d \in \delta$, and $a_d := w_{\delta(d)} \cdot \sqrt{w_{d}} \cdot b_d \cdot c$, we can use Proposition \ref{proposition:hyperball} to rewrite $V(\widetilde{C}^{\alpha})$:
\begin{align*}
V(\widetilde{C}^\alpha) &=  
\frac{1}{c^n\prod_{d \in D} w_{\delta(d)} \sqrt{w_d}}
\sum_{i=0}^{n} \left( 
\frac{(-1)^i \cdot \ln\left(\frac{\alpha}{\mu_0}\right)^i}{i!} \cdot 
\sum_{\substack{\{d_1,\dots,d_i\}\\ \subseteq D}} 
\left(\prod_{\substack{d \in \\D \setminus \{d_1,\dots,d_i\}}} a_d\right) \cdot \right.\\
&\left.\hspace{4.5cm}\prod_{\substack{\delta \in \\ \Delta_{\{d_1,\dots,d_i\}}}} \left(
|\delta|! \cdot \frac{\pi^{\frac{|\delta|}{2}}}{\Gamma\left(\frac{|\delta|}{2}+1\right)}\right)\right)
\end{align*}
We can solve Equation \ref{eqn:integral} to compute $M(\widetilde{C})$ by using the following lemma:

\begin{lemma}
\label{lemma:logIntegral}
$\forall n \in \mathbb{N}: \int_0^{1} \ln(x)^n dx = (-1)^n \cdot n!$ 
\end{lemma}
\begin{proof}
Substitute $x = e^t$ and $s = -t$, apply the definition of the $\Gamma$ function. \qed
\end{proof}

\begin{proposition}
\label{proposition:Measure}
The measure of a fuzzified cuboid $\widetilde{C}$ can be computed as follows:
\begin{align*}
M(\widetilde{C}) &= \frac{\mu_0}{c^n\prod_{d \in D} w_{\delta(d)} \sqrt{w_d}}
\sum_{i=0}^{n} \Bigg( 
\sum_{\substack{\{d_1,\dots,d_i\}\\ \subseteq D}} 
\left(\prod_{\substack{d \in \\ D \setminus \{d_1,\dots,d_i\}}} a_d\right) \cdot\\
&\hspace{5cm}\prod_{\substack{\delta \in\\ \Delta_{\{d_1,\dots,d_i\}}}} \left(
|\delta|! \cdot \frac{\pi^{\frac{|\delta|}{2}}}{\Gamma\left(\frac{|\delta|}{2}+1\right)}\right)\Bigg)
\end{align*}
\end{proposition}
\begin{proof}
Substitute $x = \frac{\alpha}{\mu_0}$ in Equation \ref{eqn:integral} and apply Lemma \ref{lemma:logIntegral}. \qed
\end{proof}

Note that although the formula for $M(\widetilde{C})$ is quite complex, it can be easily implemented via a set of nested loops. As mentioned earlier, we can use the result from Proposition \ref{proposition:Measure} in combination with the inclusion-exclusion formula (Equation \ref{eqn:inclusionExclusion}) to compute $M(\widetilde{S})$ for any concept $\widetilde{S}$. Also Equation \ref{eqn:inclusionExclusion} can be easily implemented via a set of nested loops.
Note that $M(\widetilde{S})$ is always computed only on $\Delta_S$, i.e., the set of domains on which $\widetilde{S}$ is defined.

\subsection{Subsethood}
\label{Extension:Subsethood}

In order to represent knowledge about a hierarchy of concepts, one needs to be able to determine whether one concept is a subset of another concept. For instance, the fact that $\widetilde{S}_{Granny Smith} \subseteq \widetilde{S}_{apple}$ indicates that Granny Smith is a hyponym of apple.

The classic definition of subsethood for fuzzy sets reads as follows:
$$\widetilde{S}_1 \subseteq \widetilde{S}_2 :\iff \forall {x \in CS}: \mu_{\widetilde{S}_1}(x) \leq \mu_{\widetilde{S}_2}(x)$$
We can derive the following set of necessary and jointly sufficient conditions for subsethood of concepts:

\begin{proposition}
\label{CF:proposition:CrispSubsethood}
Let $\widetilde{S}_1 = \langle S_1, \mu_0^{(1)}, c^{(1)}, W^{(1)}\rangle, \widetilde{S}_2 = \langle S_2, \mu_0^{(2)}, c^{(2)}, W^{(2)}\rangle$ be two concepts. Then, $\widetilde{S}_1 \subseteq \widetilde{S}_2$ if and only if the following conditions are true:
\begin{align*}
	\Delta_{S_2} \subseteq \Delta_{S_1} &\text{ and } \mu_0^{(1)} \leq \mu_0^{(2)} \text{ and } S_1 \subseteq {\widetilde{S}_2}^{\mu_0^{(1)}} \\
&\text{ and }\forall d \in D_{S_2}: c^{(1)} \cdot w^{(1)}_{\delta(d)} \cdot \sqrt{w^{(1)}_d} \geq c^{(2)} \cdot w^{(2)}_{\delta(d)} \cdot \sqrt{w^{(2)}_d}
\end{align*}
\end{proposition}
\begin{proof}
See Appendix B. \qed
\end{proof}

Please note that $\widetilde{S}_1 \subseteq \widetilde{S}_2$ either is true or false -- which is a binary decision.  It is however desirable to define a \emph{degree} of subsethood in order to make more fine-grained distinctions. 
Many of the definitions for degrees of subsethood proposed in the fuzzy set literature \cite{Bouchon-Meunier1996} require that the underlying universe is discrete. The following definition \cite{Kosko1992} works also in a continuous space and is conceptually quite straightforward: 
$$Sub(\widetilde{S}_1,\widetilde{S}_2) := \frac{M(\widetilde{S}_1 \cap \widetilde{S}_2)}{M(\widetilde{S}_1)} \quad \text{with a measure } M$$
One can interpret this definition intuitively as the ``percentage of $\widetilde{S}_1$ that is also in $\widetilde{S}_2$''. It can be easily implemented based on the measure defined in Section \ref{Extension:Hypervolume} and the intersection defined in \cite{Bechberger2017KI}. If $\widetilde{S}_1$ and $\widetilde{S}_2$ are not defined on the same domains, then we first project them onto their shared subset of domains before computing their degree of subsethood.

When computing the intersection of two concepts with different sensitivity parameters $c^{(1)}, c^{(2)}$ and different weights $W^{(1)}, W^{(2)}$, one needs to define new parameters $c'$ and $W'$ for the resulting concept. In our earlier work \cite{Bechberger2017KI}, we have argued that the sensitivity parameter $c'$ should be set to the minimum of $c^{(1)}$ and $c^{(2)}$. As a larger value for $c$ causes the membership function to drop faster, this means that the concept resulting from intersecting two imprecise concepts is at least as imprecise as the original concepts. Moreover, we defined the set of new salience weights $W'$ as a linear interpolation between $W^{(1)}$ and $W^{(2)}$.

Now if $c^{(1)} > c^{(2)}$, then $c' = \min(c^{(1)}, c^{(2)}) = c^{(2)} < c^{(1)}$. It might thus happen that $M(\widetilde{S}_1 \cap \widetilde{S}_2) > M(\widetilde{S}_1)$, and that therefore $Sub(\widetilde{S}_1,\widetilde{S}_2) > 1$. As we would like to confine $Sub(\widetilde{S}_1,\widetilde{S}_2)$ to the interval $[0,1]$, we should use the same $c$ and $W$ for computing both $M(\widetilde{S}_1 \cap \widetilde{S}_2)$ and $M(\widetilde{S}_1)$. 

When judging whether $\widetilde{S}_1$ is a subset of $\widetilde{S}_2$, we can think of $\widetilde{S}_2$ as setting the context by determining the relative importance of the different domains and dimensions as well as the degree of fuzziness. For instance, when judging whether tomatoes are vegetables, we focus our attention on the features that are crucial to the definition of the ``vegetable'' concept. We thus propose to use $c^{(2)}$ and $W^{(2)}$ when computing $M(\widetilde{S}_1 \cap \widetilde{S}_2)$ and $M(\widetilde{S}_1)$.

Please note that this ensures that $Sub(\widetilde{S}_1, \widetilde{S}_2) \in [0,1]$ holds. However, by using the same $c$ and $W$ for both $M(\widetilde{S}_1 \cap \widetilde{S}_2)$ and $M(\widetilde{S}_1)$, we lose the guarantee that $Sub(\widetilde{S}_1, \widetilde{S}_2) = 1.0 \Rightarrow \widetilde{S}_1 \subseteq \widetilde{S}_2$: Let $\widetilde{S} = \langle S, \mu_0, c, W \rangle$ be a concept and define $\widetilde{S}' := \langle S, \mu_0, c', W \rangle$ with $c' < c$. One can easily see that $Sub(\widetilde{S}',\widetilde{S}) = 1.0$, but that $\widetilde{S}' \not\subseteq \widetilde{S}$.
In practical applications, however, this is not necessarily a problem: If $Sub(\widetilde{S}_1, \widetilde{S}_2) = 1.0$, we can use Proposition \ref{CF:proposition:CrispSubsethood} to decide whether $\widetilde{S}_1 \subseteq \widetilde{S}_2$.

\subsection{Implication}
\label{Extension:Implication}

Implications play a fundamental role in rule-based systems and all approaches that use formal logics for knowledge representation. It is therefore desirable to define an implication function on concepts, such that one is able to express facts like $apple \Rightarrow red$ within our formalization.

In the fuzzy set literature \cite{Mas2007}, a fuzzy implication is defined as a generalization of the classical crisp implication. Computing the implication of two fuzzy sets typically results in a new fuzzy set which describes the local validity of the implication for each point in the space. In our setting, we are however more interested in a single number that indicates the overall validity of the implication $apple \Rightarrow red$.
We propose to reuse the definition of subsethood from Section \ref{Extension:Subsethood}: It makes intuitive sense in our geometric setting to say that $apple \Rightarrow red$ is true to the degree to which $apple$ is a subset of $red$. We define:
$$Impl(\widetilde{S}_1,\widetilde{S}_2) := Sub(\widetilde{S}_1,\widetilde{S}_2)$$

\subsection{Similarity}
\label{Extension:Similarity}

The similarity relation of concepts can be a valuable source for common-sense reasoning: If two concepts are similar, they are expected to have similar properties and behave in similar ways (e.g., pencils and crayons).

Whenever two concepts are defined on two different sets of domains $\Delta_1 \neq \Delta_2$, we first project them onto their set of common domains $\Delta' = \Delta_1 \cap \Delta_2$ before computing their similarity value. We define their similarity to be 0 in case that $\Delta' = \emptyset$. This is motivated by the following example: The conceptual similarity of baseball and apple should not be zero, as both have similar shapes and sizes. However, while apples have a taste, the baseball concept does not include the taste domain. Thus, when judging the similarity of ``baseball'' and ``apple'', we consider only their set of common domains.

\begin{definition}
\label{def:Similarity}
A function $Sim(\widetilde{S}_1, \widetilde{S}_2) \in [0,1]$ is called a similarity function, if it has the following properties for all concepts $\widetilde{S}_1, \widetilde{S}_2$:
\begin{enumerate}
	\item $Sim(\widetilde{S}_1, \widetilde{S}_2) = 1.0 \Rightarrow Sub(\widetilde{S}_1, \widetilde{S}_2) = 1.0$
	\item $\widetilde{S}_1 = \widetilde{S}_2 \Rightarrow Sim(\widetilde{S}_1, \widetilde{S}_2) = 1.0$
	\item $\widetilde{S}_1 \subseteq \widetilde{S}_2 \Rightarrow Sim(\widetilde{S}_1, \widetilde{S}_2) \geq Sim(\widetilde{S}_2, \widetilde{S}_1)$
\end{enumerate}
\end{definition}
The first property in Defintion \ref{def:Similarity} links perfect similarity to a strong semantic relationship (namely, subsethood) between the two concepts.
The second property requires that the similarity of a given concept to itself is always maximal.
The third property finally prevents supersets from having a higher similarity to their subsets than the other way around.\\

If we base our definition of similarity on the distance between the two concepts' cores (e.g., by computing their Hausdorff distance as proposed by \cite{Aisbett2001}, or the distance of their prototypical points as proposed by \cite{Adams2009}), we always violate the first property: Consider $\widetilde{S} = \langle S, \mu_0, c, W \rangle$ and $\widetilde{S}' = \langle S, \mu_0', c, W \rangle$ with $\mu'_0 < \mu_0$. Clearly, $Sub(\widetilde{S}, \widetilde{S}') < 1.0$. However, as the cores are identical, their distance is zero. If we use $Sim(\widetilde{S}_1, \widetilde{S}_2) = e^{-c \cdot d(S_1, S_2)}$, then $Sim(\widetilde{S}, \widetilde{S}') =1.0$, but $Sub(\widetilde{S}, \widetilde{S}') < 1.0$. 

Also if we define $Sim(\widetilde{S}_1, \widetilde{S}_2) = \max_{x \in S_1} \mu_{\widetilde{S}_2}(x)$ (or analogously with $\min$), we automatically violate the second property for $\widetilde{S}_1 = \widetilde{S}_2$ with $\mu_0 < 1.0$. 
We therefore exclude these potential definitions from our consideration.\\

The following two definitions fulfill all of our requirements:
\begin{proposition}
$Sim_S(\widetilde{S}_1, \widetilde{S}_2) := Sub(\widetilde{S}_1, \widetilde{S}_2)$ is a similarity function according to Definition \ref{def:Similarity}.
\end{proposition}
\begin{proof}
See Appendix C. \qed
\end{proof}

\begin{proposition}
$Sim_J(\widetilde{S}_1, \widetilde{S}_2) := \frac{M(\widetilde{S}_1 \cap \widetilde{S}_2)}{M(\widetilde{S}_1 \cup \widetilde{S}_2)}$ is a similarity function according to Definition \ref{def:Similarity}.
\end{proposition}
\begin{proof}
See Appendix C. \qed
\end{proof}

$Sim_S$ simply reuses our definition of subsethood from Section \ref{Extension:Subsethood} and $Sim_J$ is an implementation of the Jaccard index.
Both proposed definitions are similar in the sense that they look at the overall fuzzy sets and not just at their cores. The symmetric nature of the Jaccard index $Sim_J$ might be more convincing from a mathematical perspective. On the other hand, the asymmetric nature of $Sim_S$ matches psychological evidence suggesting that similarity judgements by humans tend to be asymmetric \cite{Tversky1977}.

\subsection{Betweenness}
\label{Extension:Betweenness}

Conceptual betweenness can be a valuable source for common-sense reasoning \cite{Derrac2015}: If one concept (e.g., ``master student'') is conceptually between two other concepts (e.g., ``bachelor student'' and ``PhD student''), then it is expected to share all properties and behaviors that the two other concepts have in common (e.g., having to pay an enrollment fee).

\cite{Derrac2014,Derrac2015,Schockaert2011a} have thoroughly studied betweenness in conceptual spaces as a basis for common-sense reasoning. They argue that betweenness is invariant under changes in context, which are typically reflected by changes of the salience weights in a conceptual space. \cite{Schockaert2011a} generalize the crisp betweenness relation from points to regions. \cite{Derrac2014} propose different soft notions of betweenness for points and subsequently generalize them to regions as well. However, they assume that each region can be described by a finite set of points. As the concepts in our formalization cannot be described by a finite set of points, their definitions are not directly applicable.\\

Please note that a concept $\widetilde{S}_2$ can only be between two other concepts $\widetilde{S}_1$ and $\widetilde{S}_3$ if all of these concepts are defined on the same domains. One can for instance not say that ``baseball'' is conceptually between ``apple'' and ``orange'', because it does not have a taste. 

As stated in Section \ref{CS:DimensionsDomainsDistance}, we can define that a point $y$ is between two other points $x$ and $z$ like this:
$$B(x,y,z) :\Leftrightarrow d(x,y) + d(y,z) = d(x,z)$$
\cite{Schockaert2011a} generalize $B(x,y,z)$ from points to sets:
$$B(S_1, S_2, S_3) :\Leftrightarrow \forall {y \in S_2}: \exists {x \in S_1}: \exists {z \in S_3}: B(x,y,z)$$
In order to generalize from crisp to fuzzy sets, we can simply require that $B(\widetilde{S}_1^{\alpha}, \widetilde{S}_2^{\alpha}, \widetilde{S}_3^{\alpha})$ is true for all $\alpha$-cuts:
\begin{align*}
B(\widetilde{S}_1, \widetilde{S}_2, \widetilde{S}_3) :&\Leftrightarrow \forall {\alpha \in [0,1]}: B(\widetilde{S}_1^{\alpha}, \widetilde{S}_2^{\alpha}, \widetilde{S}_3^{\alpha})\\
&\Leftrightarrow \forall {\alpha \in [0,1]}: \forall {y \in \widetilde{S}_2^{\alpha}}: \exists {x \in \widetilde{S}_1^{\alpha}}: \exists {z \in \widetilde{S}_3^{\alpha}}: B(x,y,z)
\end{align*}
If $\widetilde{S}_2^{\alpha} = \emptyset$, then $B(\widetilde{S}_1^{\alpha}, \widetilde{S}_2^{\alpha}, \widetilde{S}_3^{\alpha})$ is true independent of $\widetilde{S}_1^{\alpha}$ and $\widetilde{S}_3^{\alpha}$. If $\widetilde{S}_2^{\alpha} \neq \emptyset$, but $\widetilde{S}_1^{\alpha} = \emptyset$ or $\widetilde{S}_3^{\alpha} = \emptyset$, then $B(\widetilde{S}_1^{\alpha}, \widetilde{S}_2^{\alpha}, \widetilde{S}_3^{\alpha})$ is false, and therefore also $B(\widetilde{S}_1, \widetilde{S}_2, \widetilde{S}_3)$.

This definition is binary and thus allows only for relatively coarse-grained distinctions. In order to derive a \emph{degree} of betweenness for fuzzy sets, we use the following soft notion of betweenness for points \cite{Derrac2014}:
$$B_{soft}(x,y,z) := \frac{d(x,z)}{d(x,y) + d(y,z)}$$
Please note that $B(x,y,z) \Leftrightarrow B_{soft}(x,y,z) = 1.0$ and that $B_{soft}(x,y,z) \in [0,1]$. 

We can use $B_{soft}(x,y,z)$ together with the extension principle \cite{Zadeh1975} to generalize $B(\widetilde{S}_1, \widetilde{S}_2, \widetilde{S}_3)$ to a soft notion $B_{soft}(\widetilde{S}_1, \widetilde{S}_2, \widetilde{S}_3)$:
$$B_{soft}(\widetilde{S}_1, \widetilde{S}_2, \widetilde{S}_3) := \min_{\alpha \in [0,1]} \min_{y \in \widetilde{S}_2^{\alpha}} \max_{x \in \widetilde{S}_1^{\alpha}} \max_{z \in \widetilde{S}_3^{\alpha}} B_{soft}(x,y,z)$$
We simply replaced $B(x,y,z)$ with $B_{soft}(x,y,z)$, each existence quantor with a $\max$, and each all quantor with a $\min$. One can easily see that $B(\widetilde{S}_1, \widetilde{S}_2, \widetilde{S}_3) \Leftrightarrow B_{soft}(\widetilde{S}_1, \widetilde{S}_2, \widetilde{S}_3) = 1.0$.
Moreover, if $\widetilde{S}_2 \subseteq \widetilde{S}_1$, then $B_{soft}(\widetilde{S}_1, \widetilde{S}_2, \widetilde{S}_3) = 1.0$ as we can pick for each $y \in \widetilde{S}_2^{\alpha}$ always $x = y \in \widetilde{S}_1^{\alpha}$.

\begin{figure}[tp]
\centering
\includegraphics[width=\textwidth]{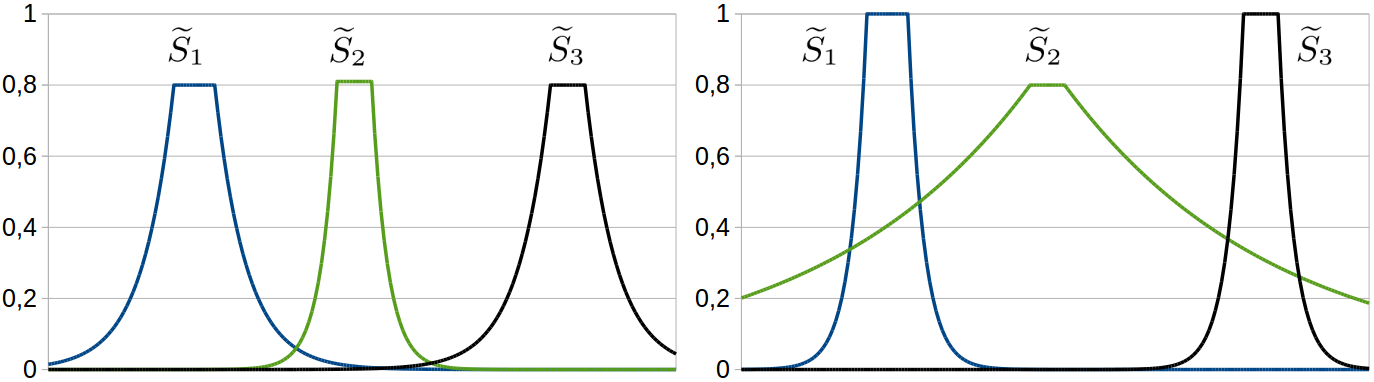}\\
(a)\hspace{5.5cm}(b)

\caption{Two problematic cases for $B_{soft}(\widetilde{S}_1, \widetilde{S}_2, \widetilde{S}_3)$.}
\label{fig:BetweennessProblems}
\end{figure}

Unfortunately, this definition yields some unintutive results: If we consider the three concepts in Figure \ref{fig:BetweennessProblems}a, where $\mu_0^{(1)} = \mu_0^{(3)} = 0.80$ and $\mu_0^{(2)} = 0.81$, then we get that $B_{soft}(\widetilde{S}_1, \widetilde{S}_2, \widetilde{S}_3) = 0.0$, as for $\alpha = 0.81$, $\widetilde{S}_1^{\alpha} = \widetilde{S}_3^{\alpha} = \emptyset$. Also in the example shown in Figure \ref{fig:BetweennessProblems}b, we get $B_{soft}(\widetilde{S}_1, \widetilde{S}_2, \widetilde{S}_3) = 0.0$: For $\alpha \rightarrow 0$, $\widetilde{S}_2^{\alpha}$ becomes much larger than $\widetilde{S}_1^{\alpha}$ and $\widetilde{S}_3^{\alpha}$ and we can select a point $y \in \widetilde{S}_2^{\alpha}$ for which $B_{soft}(x,y,z)$ becomes arbitrarily small independent of the choices for $x$ and $z$.

In order to achieve a more generous degradation in cases like the ones depicted in Figure \ref{fig:BetweennessProblems}, we propose to aggregate over the $\alpha$-cuts not by taking the minimum, but through an integration. As we integrate over $\alpha$ in the interval $[0,1]$ and as the degree of betweenness computed for each $\alpha$-cut also lies in the interval $[0,1]$, the result of this integration is also a number between zero and one. We define:
$$B_{soft}^{integral}(\widetilde{S}_1, \widetilde{S}_2, \widetilde{S}_3) := \int_0^1 \min_{y \in \widetilde{S}_2^{\alpha}} \max_{x \in \widetilde{S}_1^{\alpha}} \max_{z \in \widetilde{S}_3^{\alpha}} \frac{d(x,z)}{d(x,y) + d(y,z)} \; d\alpha$$

All of the properties discovered above for $B_{soft}(\widetilde{S}_1, \widetilde{S}_2, \widetilde{S}_3)$ still hold for $B_{soft}^{integral}(\widetilde{S}_1, \widetilde{S}_2, \widetilde{S}_3)$. However, in the examples from Figure \ref{fig:BetweennessProblems} we do not get the unintuitive result of $B_{soft}^{integral}(\widetilde{S}_1, \widetilde{S}_3, \widetilde{S}_3) = 0.0$.

%================================================================================================================================================================%
\section{Illustrative Example}
\label{Example}

%---------------------------------------------------------------------------------------------------------------%
\subsection{A Conceptual Space and its Concepts}
\label{Example:Definition}

We consider a very simplified conceptual space for fruits, consisting of the following domains and dimensions:
$$\Delta = \left\{ \delta_{color} = \left\{d_{hue}\right\},\delta_{shape} = \left\{d_{round}\right\},\delta_{taste} = \left\{d_{sweet}\right\} \right\}$$
$d_{hue}$ describes the hue of the observation's color, ranging from $0.00$ (purple) to $1.00$ (red). $d_{round}$ measures the percentage to which the bounding circle of an object is filled. $d_{sweet}$ represents the relative amount of sugar contained in the fruit, ranging from 0.00 (no sugar) to 1.00 (high sugar content). As all domains are one-dimensional, the dimension weights $w_{d}$ are always equal to 1.00 for all concepts. We assume that the dimensions are ordered like this: $d_{hue},d_{round},d_{sweet}$. Table \ref{tab:FruitSpace} defines several concepts in this space and Figure \ref{fig:FruitSpace} visualizes them.

\begin{table}[t]
  \centering
  \caption{Definitions of several concepts.}
  \begin{tabular}{|l||c|c|c|c|c|c|c|c|}
    \hline
    Concept 	& $\Delta_S$& $p^-$ 				& $p^+$ 				& $\mu_0$ 	& $c$ 	& \multicolumn{3}{|c|}{$W$}\\ %\hline 
    & & & & & & $w_{\delta_{color}}$ & $w_{\delta_{shape}}$ & $w_{\delta_{taste}}$\\ \hline \hline
    Orange		& $\Delta$	& (0.80, 0.90, 0.60)	& (0.90, 1.00, 0.70)	& 1.0		& 15.0	& 1.00 & 1.00 & 1.00 \\ \hline
    Lemon		& $\Delta$	& (0.70, 0.45, 0.00)	& (0.80, 0.55, 0.10)	& 1.0		& 20.0	& 0.50 & 0.50 & 2.00 \\ \hline
    Granny & \multirow{2}{*}{$\Delta$}	& \multirow{2}{*}{(0.55, 0.70, 0.35)}	& \multirow{2}{*}{(0.60, 0.80, 0.45)}	& \multirow{2}{*}{1.0}		& \multirow{2}{*}{25.0} & \multirow{2}{*}{1.00} & \multirow{2}{*}{1.00} &  \multirow{2}{*}{1.00}\\ 
    Smith & & & & & & & &\\ \hline
    \multirow{3}{*}{Apple}	& \multirow{3}{*}{$\Delta$}	& (0.50, 0.65, 0.35)	& (0.80, 0.80, 0.50)	& \multirow{3}{*}{1.0}		& \multirow{3}{*}{10.0}	& \multirow{3}{*}{0.50} & \multirow{3}{*}{1.50} &  \multirow{3}{*}{1.00} \\ %\hline
    			&			& (0.65, 0.65, 0.40)	& (0.85, 0.80, 0.55)	&			&		&  & & \\ %\hline
    			&			& (0.70, 0.65, 0.45)	& (1.00, 0.80, 0.60)	&			&		&  & & \\ \hline
    Red			& $\{\delta_{color}\}$ & (0.90, -$\infty$, -$\infty$) & (1.00, +$\infty$, +$\infty$) & 1.0 & 20.0 & 1.00 & -- & -- \\ \hline
  \end{tabular}\\[1ex]
  \label{tab:FruitSpace}
\end{table}

\begin{figure}[t]
\centering
\includegraphics[width=\textwidth]{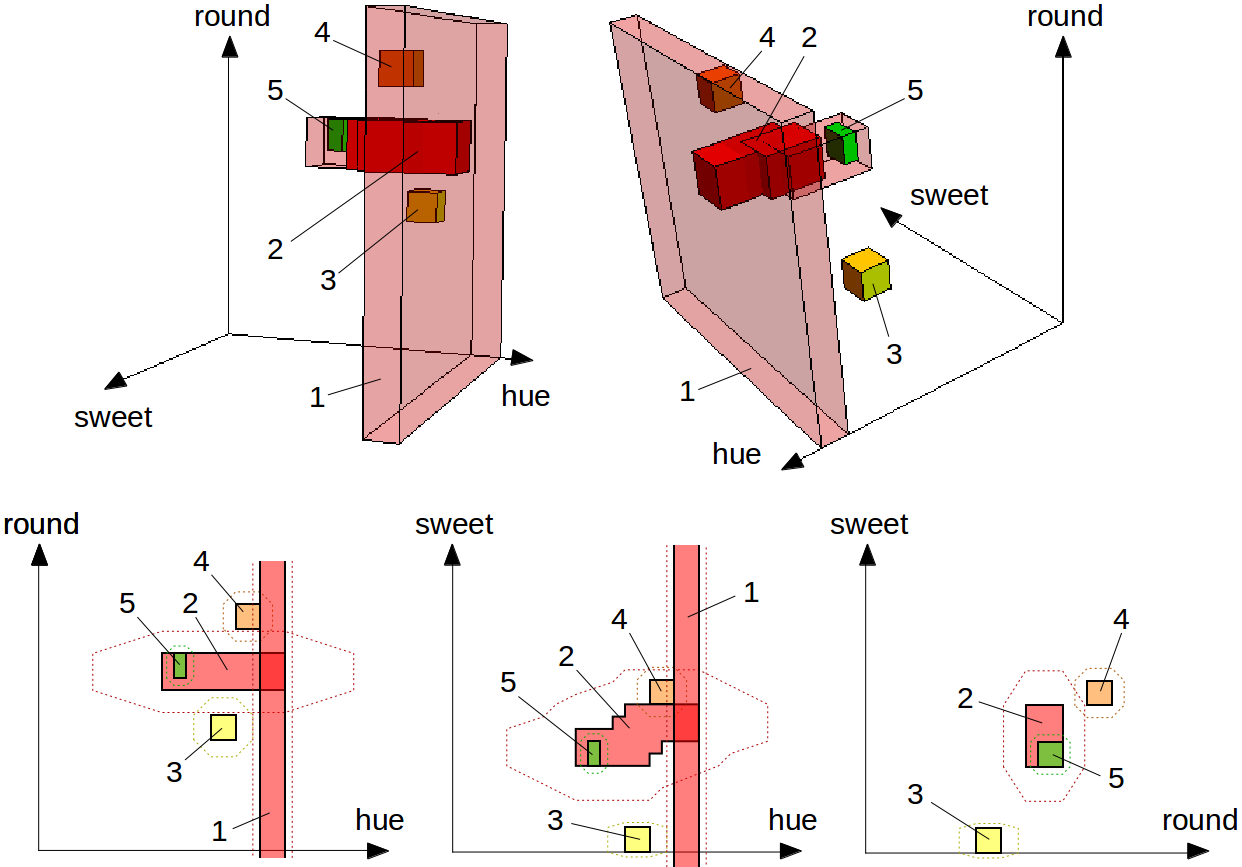}
\caption{Top: Three-dimensional visualization of the fruit space (only cores). Bottom: Two-dimensional visualizations of the fruit space (cores and 0.5-cuts). The concepts are labeled as follows: red (1), apple (2), lemon (3), orange (4), Granny Smith (5).}
\label{fig:FruitSpace}
\end{figure}

%---------------------------------------------------------------------------------------------------------------%
\subsection{Computations}
\label{Example:Computations}

\begin{table}[t]
  \centering
  \caption{Computations of different relations. Note that $Sub(\widetilde{S}_1, \widetilde{S}_2) = Impl(\widetilde{S}_1, \widetilde{S}_2) = Sim_S(\widetilde{S}_1, \widetilde{S}_2)$.}
  \begin{tabular}{|c|c||c|c|c|c|c|c|}
  \hline
  $\widetilde{S}_1$ & $\widetilde{S}_2$ & $M(\widetilde{S}_1)$ & $M(\widetilde{S}_2)$ & $Sub(\widetilde{S}_1, \widetilde{S}_2)$ & $Sub(\widetilde{S}_2, \widetilde{S}_1)$ & $Sim_J(\widetilde{S}_1, \widetilde{S}_2)$ & $Sim_J(\widetilde{S}_2, \widetilde{S}_1)$\Tstrut \\ \hline \hline
  Granny Smith 	& Apple 	& 0.0042 	& 0.1048 	& 1.0000 	& 0.1171 	& 0.2570 	& 0.2570 	\\ \hline
  Orange 		& Apple 	& 0.0127	& 0.1048	& 0.1800	& 0.0333	& 0.0414	& 0.0414	\\ \hline
  Lemon			& Apple		& 0.0135	& 0.1048	& 0.0422	& 0.0054	& 0.0073	& 0.0073	\\ \hline
  Red			& Apple		& 0.2000	& 0.1048	& 1.0000	& 0.3333	& 0.4286	& 0.4286	\\ \hline
  \end{tabular}\\[2ex]
  \begin{tabular}{|c|c|c||c|c|}
  \hline
  $\widetilde{S}_1$ & $\widetilde{S}_2$ & $\widetilde{S}_3$ & $B_{soft}(\widetilde{S}_1, \widetilde{S}_2, \widetilde{S}_3)$ & $B_{soft}^{integral}(\widetilde{S}_1, \widetilde{S}_2, \widetilde{S}_3)$ \Tstrut \\ \hline \hline
  Lemon			& Apple			& Orange	& 0.0000 	& 0.8623 \\ \hline
  Lemon			& Granny Smith	& Orange	& 0.8254	& 0.9161 \\ \hline
  Granny Smith	& Apple			& Red		& 0.0000	& 0.0000 \\ \hline
  Apple			& Granny Smith	& Orange	& 1.0000	& 1.0000 \\ \hline
  \end{tabular}\\[1ex]
  \label{tab:computations}
\end{table}

Table \ref{tab:computations} shows the results of using the definitions from Section \ref{Extension} on the concepts defined in Section \ref{Example:Definition}. Note that $M(\widetilde{S}_{lemon}) \neq M(\widetilde{S}_{orange})$ because the two concepts have different weights and different sensitivity parameters. Also all relations involving the property ``red'' tend to yield relatively high numbers -- this is because all computations only take place within the single domain on which ``red'' is defined. The numbers computed for the subsethood/implication relation nicely reflect our intuitive expectations. The numbers also nicely illustrate the difference between the asymmetric $Sim_S$ and the symmetric $Sim_J$ similarity functions. The values of $Sim_J(\widetilde{S}_1, \widetilde{S}_2)$ are always between $Sim_S(\widetilde{S}_1, \widetilde{S}_2)$ and $Sim_S(\widetilde{S}_2, \widetilde{S}_1)$ while tending towards the smaller of the two numbers. Please note that for none of the examples we get a value for $Sim_J$ that is close to one. The numbers derived for the two betweenness relations show that $B_{soft}^{integral}$ is always greater or equal than $B_{soft}$. They agree on borderline cases: When computing the betweenness of concepts defined on different sets of domains (e.g., Granny Smith, apple, and red) they return zero, and when the middle concept is a subset of one of the outer concepts (e.e.g, apple, Granny Smith, and orange), they return one. The first example (lemon, apple, orange) shows the more generous degradation of $B_{soft}^{integral}$ in a case similar to the one from Figure \ref{fig:BetweennessProblems}b.

%================================================================================================================================================================%
\section{Related Work}
\label{RelatedWork}

Our work is of course not the first attempt to devise an implementable formalization of the conceptual spaces framework.

%Aisbett2001
An early formalization was done by \cite{Aisbett2001}. Like we, they consider concepts to be regions in the overall conceptual space. However, they stick with the assumption of convexity and do not define concepts in a parametric way. The only operations they provide are distance and similarity of points and regions. Their formalization targets the interplay of symbols and geometric representations, but it is too abstract to be implementable. 
% dist & sim of objects & regions (Hausdorff); classification

%Rickard2006
\cite{Rickard2006} provides a formalization based on fuzziness. He represents concepts as co-occurence matrices of their properties. By using some mathematical transformations, he interprets these matrices as fuzzy sets on the universe of ordered property pairs. Operations defined on these concepts include similarity judgements between concepts and between concepts and instances. Rickard's representation of correlations is not geometrical: He first discretizes the domains (by defining properties) and then computes the co-occurences between these properties. Depending on the discretization, this might lead to a relatively coarse-grained notion of correlation. Moreover, as properties and concepts are represented in different ways, one has to use different learning and reasoning mechanisms for them. His formalization is also not easy to work with due to the complex mathematical transformations involved.
% concept similarity & observation-concept similarity; (subsethood)

%Adams2009
\cite{Adams2009} represent concepts by one convex polytope per domain. This allows for efficient computations while being potentially more expressive than our cuboid-based approach. The Manhattan metric is used to combine different domains. However, correlations between different domains are not taken into account and cannot be expressed in this formalization as each convex polytope is only defined on a single domain. Adams and Raubal also define operations on concepts, namely intersection, similarity computation, and concept combination. This makes their formalization quite similar in spirit to ours. 
% intersectino of regions; membership point in region; distance within a domain; instance similarity, concept similarity; concept combination (property-concept, concept-concept, relativeProperty-concept)

%Lewis2016
\cite{Lewis2016} formalize conceptual spaces using random set theory. They define properties as random sets within single domains and concepts as random sets in a boolean space whose dimensions indicate the presence or absence of properties. In order to define this boolean space, a single property is taken from each domain. Their approach is similar to ours in using a distance-based membership function to a set of prototypical points. However, their work purely focuses on modeling conjunctive concept combinations and does not consider correlations between domains. 
% concept combination (prop-prop); conjunction of compound concepts (property-concept, concept-concept)

As one can see, none of the formalizations listed above provides a set of operations that is as comprehensive as the one offered by our extended formalization.
%================================================================================================================================================================%
\section{Conclusion and Future Work}
\label{Conclusion}

In this paper, we extended our previous formalization of the conceptual spaces framework by providing ways to measure relations between concepts: concept size, subsethood, implication, similarity, and betweenness. This considerably extends our framework's capabilities for representing knowledge and makes it (to the best of our knowledge) the most thorough and comprehensive formalization of conceptual spaces developed so far. A python implementation of our formalization is publicly available on GitHub under \url{https://www.github.com/lbechberger/ConceptualSpaces} \cite{Bechberger2018GitHub1.2}.

Our overall research goal is to use machine learning in conceptual spaces. Therefore, our future work will focus on putting this formalization to practical use in a machine learning context.

% The file named.bst is a bibliography style file for BibTeX 0.99c
\bibliographystyle{apalike}
\bibliography{/home/lbechberger/Documents/Papers/jabref.bib}

\appendix

\section{Concept Size}

\subsection{Hyperballs under the Unweighted Metric}
\label{Size:Hyperballs}

In general, a hyperball of radius $r$ around a point $p$ can be defined as the set of all points with a distance of at most $r$ to $p$:
$$H =  \{x \in CS \;|\;d(x,p) \leq r \}$$

\begin{figure}[tp]
\centering
\includegraphics[width=\textwidth]{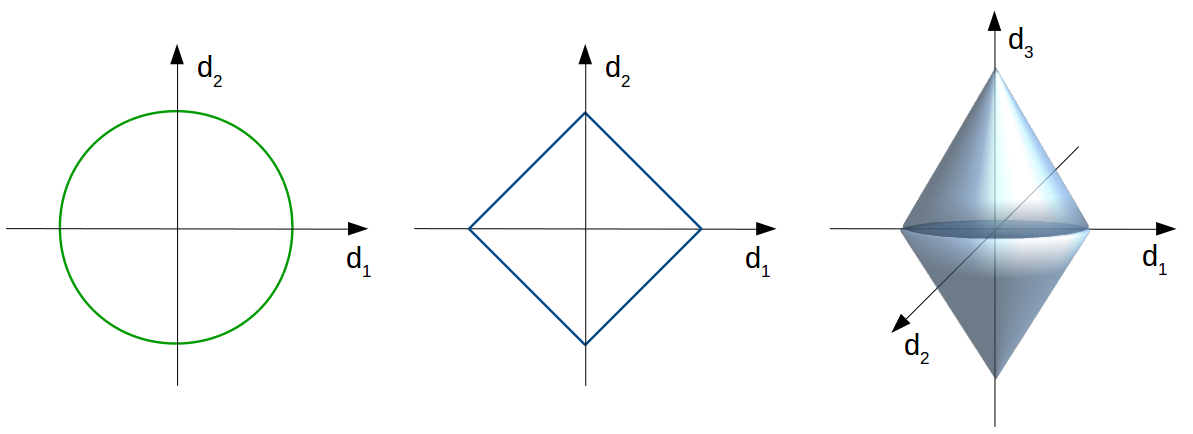}
\caption{Left: Two-dimensional hyperball under the Euclidean metric. Middle: Two-dimensional hyperball under the Manhattan metric. Right: Three-dimensional hyperball under the combined metric (with domain structure $\Delta = \{ \delta_1 = \{d_1, d_2\}, \delta_2 = \{d_3\}\}$).}
\label{fig:hyperballs}
\end{figure}

If the Euclidean distance $d_E$ is used, this corresponds to our intuitive notion of a ball -- a round shape centered at $p$. However, under the Manhattan distance $d_M$, hyperballs have the shape of diamonds. Under the combined distance $d_C^\Delta$, a hyperball in three dimensions has the shape of a double cone (cf. Figure \ref{fig:hyperballs}).

As similarity is inversely related to distance, one can interpret a hyperball in a conceptual space as the set of all points that have a minimal similarity $\alpha$ to the central point $p$, where $\alpha$ depends on the radius of the hyperball.

In this section, we assume an unweighted version of $d_C^\Delta$:
$$d_C^\Delta(x,y) = \sum_{\delta \in \Delta} \sqrt{\sum_{d \in \delta} |x_{d} - y_{d}|^2}$$

In order to derive a formula for the hypervolume of a hyperball under $d_C^\Delta$, we need to use the following three lemmata:

\begin{lemma}
\label{lemma:AngleIntegral}
$$\int_0^{2\pi} \int_0^\pi \int_0^\pi \dots \int_0^\pi \sin^{n-2}(\phi_1) \sin^{n-3}(\phi_2) \dots \sin(\phi_{n-2}) d\phi_1 \dots d\phi_{n-1} = 2 \cdot \frac{\pi^{\frac{n}{2}}}{\Gamma\left(\frac{n}{2}\right)}$$
where $\Gamma(\cdot)$ is Euler's Gamma function and $n \in \mathbb{N}$.
\end{lemma}
\begin{proof}
\begin{align*}
I &=\int_0^{2\pi} \int_0^\pi \int_0^\pi \dots \int_0^\pi \sin^{n-2}(\phi_1) \sin^{n-3}(\phi_2) \dots \sin(\phi_{n-2}) d\phi_1 \dots d\phi_{n-1}\\
&= \left(\int\displaylimits_{0}^{2\pi} 1 \;d\phi_{n - 1}\right) 
\left(\int\displaylimits_{0}^\pi \sin(\phi_{n-2}) \;d\phi_{n - 2}\right)
\cdots 
\left(\int\displaylimits_{0}^\pi \sin^{n-2}(\phi_{1}) \;d\phi_{1}\right)\\
&= \left(4 \cdot \int\displaylimits_{0}^{\frac{\pi}{2}} 1 \;d\phi_{n - 1}\right) 
\left(2 \cdot \int\displaylimits_{0}^{\frac{\pi}{2}} \sin(\phi_{n-2}) \;d\phi_{n-2}\right)
\cdots 
\left(2 \cdot \int\displaylimits_{0}^{\frac{\pi}{2}} \sin^{n-2}(\phi_{1}) \;d\phi_{1}\right)
\end{align*}
We can now use the definition of the Beta function, which is $$B(x,y) = 2 \cdot \int\displaylimits_{0}^{\frac{\pi}{2}} \sin^{2x-1}(\phi) \cos^{2y-1}(\phi)\; d\phi$$
Using $y = \frac{1}{2}$, we get:
\begin{align*}
I &= \left(4 \cdot \int\displaylimits_{0}^{\frac{\pi}{2}} 1 \;d\phi_{n - 1}\right) 
\left(2 \cdot \int\displaylimits_{0}^{\frac{\pi}{2}} \sin(\phi_{n-2}) \;d\phi_{n-2}\right)
\cdots 
\left(2 \cdot \int\displaylimits_{0}^{\frac{\pi}{2}} \sin^{n-2}(\phi_{1}) \;d\phi_{1}\right)\\
&= 2 \cdot B\left(\frac{1}{2},\frac{1}{2}\right) \cdot B\left(1,\frac{1}{2}\right) \cdots B\left(\frac{n-2}{2},\frac{1}{2}\right) \cdot B\left(\frac{n-1}{2},\frac{1}{2}\right)
\end{align*}
Next, we use the identity $B(x,y) = \frac{\Gamma(x)\Gamma(y)}{\Gamma(x+y)}$ with Euler's Gamma function $\Gamma$ and the fact that $\Gamma(\frac{1}{2}) = \sqrt{\pi}$. We get:
$$I = 2 \cdot \frac{\Gamma(\frac{1}{2})\Gamma(\frac{1}{2})}{\Gamma(1)} \cdot \frac{\Gamma(1)\Gamma(\frac{1}{2})}{\Gamma(\frac{3}{2})} \cdots \frac{\Gamma(\frac{n-2}{2})\Gamma(\frac{1}{2})}{\Gamma(\frac{n-1}{2})} \cdot \frac{\Gamma(\frac{n-1}{2})\Gamma(\frac{1}{2})}{\Gamma(\frac{n}{2})} 
= 2 \cdot \frac{\Gamma(\frac{1}{2})^{n}}{\Gamma(\frac{n}{2})} 
= 2 \cdot \frac{\pi^{\frac{n}{2}}}{\Gamma(\frac{n}{2})}$$ \qed
\end{proof}

\begin{lemma}
\label{lemma:IntegralBetaFunction}
For any natural number $j > 0$ and any $a,b \in \mathbb{R}$, the following equation holds:
$$\int_0^{r-\sum_{i=1}^{j-1}r_i} r_j^{a-1} \cdot \left(r- \sum_{i=1}^j r_i\right)^b dr_j = B(a,b+1) \cdot \left(r- \sum_{i=1}^{j-1} r_i\right)^{a+b}$$
\end{lemma}
\begin{proof}
We can make a variable change by defining $r_j = \left(r- \sum_{i=1}^{j-1} r_i\right)\cdot z$ which gives $dr_j = \left(r- \sum_{i=1}^{j-1} r_i\right)\cdot dz$. This results in:
\begin{align*}
&\int_0^{r-\sum_{i=1}^{j-1}r_i} r_j^{a-1} \cdot \left(r- \sum_{i=1}^j r_i\right)^b dr_j\\
&=\int_0^{1} \left(r- \sum_{i=1}^{j-1} r_i\right)^{a-1} \cdot z^{a-1} \cdot \left(r- \sum_{i=1}^{j-1} r_i - \left(r- \sum_{i=1}^{j-1} r_i\right)\cdot z\right)^b \cdot\left(r- \sum_{i=1}^{j-1} r_i\right) dz\\
&= \left(r- \sum_{i=1}^{j-1} r_i\right)^{a-1+b+1} \int_0^{1} z^{a-1} (1-z)^b dz = \left(r- \sum_{i=1}^{j-1} r_i\right)^{a+b} \cdot B(a,b+1)
\end{align*}
The last transformation uses the fact that $B(x,y) = \int_0^1 t^{x-1} (1-t)^{y-1} dt$. \qed
\end{proof}

\begin{lemma}
\label{lemma:RadiusIntegral}
For any natural number $k > 0$, any $r_1,\dots,r_k,n_1,\dots,n_k > 0$, $n = \sum_{i=0}^k n_i$, $r = \sum_{i=0}^k r_i$, the following equation holds:
$$\int_0^r r_1^{n_1 - 1} \int_0^{r-r_1} r_2^{n_2 - 1} \dots \int_0^{r-\sum_{i=1}^{k-1} r_i} r_k^{n_k - 1} dr_k \dots dr_1 = \frac{r^n}{\Gamma(n+1)} \prod_{i=1}^k \Gamma(n_i)$$
\end{lemma}
\begin{proof}

Using Lemma \ref{lemma:IntegralBetaFunction}, we can solve the innermost integral by setting $j=k, a= n_k, b=0$, which gives us $B(n_k,1) \cdot \left(r- \sum_{i=1}^{k-1} r_i\right)^{n_k}$. Therefore:
\begin{align*}
I &= \int_0^r r_1^{n_1 - 1} \int_0^{r-r_1} r_2^{n_2 - 1} \dots \int_0^{r-\sum_{i=1}^{k-1} r_i} r_k^{n_k - 1} dr_k \dots dr_1\\
&= B(n_k,1) \cdot \int_0^r r_1^{n_1 - 1} \dots \int_0^{r-\sum_{i=1}^{k-2} r_i} r_{k-1}^{n_{k-1} - 1} \cdot \left(r- \sum_{i=1}^{k-1} r_i\right)^{n_k} dr_{k-1} \dots dr_1
\end{align*}
As one can see, we can again apply Lemma \ref{lemma:IntegralBetaFunction} to the innermost integral. Repeatedly applying Lemma \ref{lemma:IntegralBetaFunction} finally results in:
$$I = B(n_k,1)\cdot B(n_{k-1},n_k + 1) \cdot \dots \cdot B(n_1, n_2+\dots+n_k+1)\cdot r^{n_1+\dots+n_k}$$
We use that $B(x,y) = \frac{\Gamma(x)\Gamma(y)}{\Gamma(x+y)}$ in order to rewrite this equation:
$$I = r^{n_1+\dots+n_k} \cdot \frac{\Gamma(n_k)\Gamma(1)}{\Gamma(n_k+1)} \cdot \frac{\Gamma(n_{k-1})\Gamma(n_k+1)}{\Gamma(n_{k-1}+n_k+1)} \cdots \frac{\Gamma(n_1)\Gamma(n_2+\dots+n_k+1)}{\Gamma(n_1+n_2+\dots+n_k+1)}$$
Because $\Gamma(1) = 1$ and because most of the terms cancel out, this reduces to:
$$I = r^{n_1+\dots+n_k} \cdot \Gamma(n_k) \cdots \Gamma(n_1) \cdot \frac{1}{\Gamma(n_1+\dots+n_k+1)} = \frac{r^n}{\Gamma(n+1)} \prod_{i=1}^{k} \Gamma(n_i)$$ \qed
\end{proof}
Using these three lemata, we can now derive the size of a hyperball in a conceptual space without domain and dimension weights: 

\begin{lemma}
\label{lemma:HyperballVolume}
The hypervolume of a hyperball with radius $r$ in a space with the unweighted combined metric $d_C^{\Delta}$ and the domain structure $\Delta$ can be computed in the following way, where $n$ is the overall number of dimensions and $n_\delta$ is the number of dimensions in domain $\delta$:
$$V(r,\Delta) = \frac{r^n}{n!} \prod_{\delta \in \Delta} \left(n_\delta! \frac{\pi^{\frac{n_\delta}{2}}}{\Gamma\left(\frac{n_\delta}{2}+1\right)}\right)$$
\end{lemma}
\begin{proof}

The hyperball can be defined as the set of  all points that have a distance of maximally $r$ to the origin, i.e., $$H = \Big \{x \in CS \;|\;d_C^\Delta(x,0) = \sum_{\delta \in \Delta} \sqrt{\sum_{d \in \delta} x_{d}^2} \leq r\Big \}$$
If we define $\forall \delta \in \Delta: r_\delta := \sqrt{\sum_{d \in \delta} x_{d}^2}$, we can easily see that $\sum_{\delta \in \Delta} r_\delta \leq r$. The term $r_\delta$ can be interpreted as the distance between $x$ and the origin within the domain $\delta$. The constraint $\sum_{\delta \in \Delta} r_\delta \leq r$ then simply means that the sum of domain-wise distances is less than the given radius. One can thus interpret $r_\delta$ as the radius within domain $\delta$.

We would ultimately like to compute
$$V(r,\Delta) = \int \dots \int_H 1\; dH$$ 
This integration becomes much easier if we use spherical coordinates instead of the Cartesian coordinates provided by our conceptual space.\\

Let us first consider the case of a single domain $\delta$ of size $n_\delta$. A single domain corresponds to a standard Euclidean space, therefore we can use the standard procedure of changing to spherical coordinates (cf., e.g., \cite{DeRise1992}). Let us index the dimensions of $\delta$ as $d_1,\dots,d_{n_\delta}$. The coordinate change within the domain $\delta$ then looks like this:
\begin{align*}
x_{1} &= t \cdot \cos(\phi_{1})\\
x_{2} &= t \cdot \sin(\phi_{1}) \cdot \cos(\phi_{2})\\
&\hspace{0.2cm}\vdots\\
x_{{n_\delta}-1} &= t \cdot \sin(\phi_{1}) \cdots \sin(\phi_{{n_\delta}-2}) \cdot \cos(\phi_{{n_\delta}-1})\\
x_{{n_\delta}} &= t \cdot \sin(\phi_{1}) \cdots \sin(\phi_{{n_\delta}-2}) \cdot \sin(\phi_{{n_\delta}-1})
\end{align*}

In order to switch the integral to spherical coordinates, we need to calucalate the volume element. This can be found by looking at the determinant of the transformation's Jacobian matrix. The Jacobian matrix of the transformation of a single domain $\delta$ can be written as follows:

\begin{gather*}
 J_{\delta} =
  \left[ {\begin{array}{cccc}
   \frac{\delta x_{1}}{\delta t} & \frac{\delta x_{1}}{\delta \phi_{1}} & \dots & \frac{\delta x_{1}}{\delta \phi_{{n_\delta}-1}} \\
   \frac{\delta x_{2}}{\delta t} & \frac{\delta x_{2}}{\delta \phi_{1}} & \dots & \frac{\delta x_{2}}{\delta \phi_{{n_\delta}-1}} \\
   \vdots & \vdots & \ddots & \vdots \\
   \frac{\delta x_{{n_\delta}}}{\delta t} & \frac{\delta x_{{n_\delta}}}{\delta \phi_{1}} & \dots & \frac{\delta x_{{n_\delta}}}{\delta \phi_{{n_\delta}-1}} \\
  \end{array} } \right] =\\[1ex]
  \hspace{-0.2cm}\left[ {\tiny\begin{array}{cccccc}
	\cos(\phi_{1}) 							&-t\sin(\phi_{1}) 						&0 &0 &\dots &0 \\
	\sin(\phi_{1})\cos(\phi_{2})		&t\cos(\phi_{1})\cos(\phi_{2})	&-t\sin(\phi_{1})\sin(\phi_{2}) &0   &\dots &0\\
	\vdots &\vdots &\vdots &\vdots &\vdots &\vdots\\
	\sin(\phi_{1})\cdots\sin(\phi_{{n_\delta}-2})\cos(\phi_{{n_\delta}-1}) &\dots &\dots &\dots &\dots &-t\sin(\phi_{1})\cdots\sin(\phi_{{n_\delta}-2})\sin(\phi_{{n_\delta}-1})\\
	\sin(\phi_{1})\cdots\sin(\phi_{{n_\delta}-2})\sin(\phi_{{n_\delta}-1}) &\dots &\dots &\dots &\dots &t\sin(\phi_{1})\cdots\sin(\phi_{{n_\delta}-2})\cos(\phi_{{n_\delta}-1})
  \end{array} } \right]
\end{gather*}
The determinant of this matrix can be computed like this:
$$det(J_{\delta}) = t^{{n_\delta}-1}\cdot\sin^{{n_\delta}-2}(\phi_1)\cdot\sin^{{n_\delta}-3}(\phi_2)\cdots\sin(\phi_{{n_\delta}-2})$$
We can now perform the overall switch from Cartesian to spherical coordinates by performing this coordinate change for each domain individually. Let us index the Cartesian coordinates of a point $x$ in domain $\delta$ by $x_{\delta,1},\dots,x_{\delta,n_\delta}$. Let us further index the spherical coordinates of domain $\delta$ by $r_\delta$ and $\phi_{\delta,1},\dots,\phi_{\delta,n_\delta-1}$. Let $k = |\Delta|$ denote the total number of domains.

Because $x_{\delta,j}$ is defined independently from $r_{\delta'}$ and $\phi_{\delta',j'}$ for different domains $\delta \neq \delta'$, any derivative $\frac{x_{\delta,j}}{r_{\delta'}}$ or $\frac{x_{\delta,j}}{\phi_{\delta',j'}}$ will be zero. If we apply the coordinate change to all domains at once, the Jacobian matrix of the overall transformation has therefore the structure of a block matrix:

$$
   J=
  \left[ {\begin{array}{cccc}
   J_1 & 0 & \dots & 0 \\
   0 & J_2 & \dots & 0 \\
   \vdots & \vdots & \ddots & \vdots \\
   0 & 0 & \dots & J_k \\
  \end{array} } \right]
$$

The blocks on the diagonal are the Jacobian matrices of the individual domains as defined above, and all other blocks are filled with zeroes because all cross-domain derivatives are zero.
Because the overall $J$ is a block matrix, we get that $det(J) = \prod_{\delta \in \Delta} det(J_\delta)$ (cf. \cite{Silvester2000}).
Our overall volume element is thus
$$det(J) = \prod_{\delta \in \Delta} det(J_\delta) = \prod_{\delta \in \Delta} r_\delta^{n_\delta-1}\sin^{n_\delta-2}(\phi_{\delta,1})\sin^{n_\delta-3}(\phi_{\delta,2})\cdots \sin(\phi_{\delta,n_\delta-2})$$

The limits of the angle integrals are $[0,2\pi]$ for the outermost and $[0,\pi]$ for all other integrals.
Based on our constraint $\sum_{\delta \in \Delta} r_\delta \leq r$, we can derive the limits for the integrals over the $r_\delta$ as follows, assuming an arbitrarily ordered indexing $\delta_1,\dots,\delta_k$ of the domains:
\begin{align*}
r_1 &\in [0,r]\\
r_2 &\in [0,r - r_1]\\
r_3 &\in [0,r - r_1 - r_2]\\
&\hspace{0.2cm}\vdots\\
r_k &\in [0,r - \sum_{i=1}^{k-1} r_i]
\end{align*}
The overall coordinate change therefore looks like this:
\begin{align*}
V&(r,\Delta) = \int \dots \int_H 1\; dH \\
&=
\underbrace{\int\displaylimits_{\phi_{1,n_1-1}=0}^{2\pi} \int\displaylimits_{\phi_{1,n_1-2}=0}^\pi
\cdots \int\displaylimits_{\phi_{1,1}=0}^\pi \int\displaylimits_{r_1 = 0}^r}_{\delta = 1} 
\cdots
\underbrace{\int\displaylimits_{\phi_{k,n_k-1}=0}^{2\pi} \int\displaylimits_{\phi_{k,n_k-2}=0}^\pi
\cdots \int\displaylimits_{\phi_{k,1}=0}^\pi \int\displaylimits_{r_k = 0}^{r - \sum_{i=1}^{k-1} r_i}}_{\delta = k}\\[2ex]
&\hspace{0.8cm}\underbrace{r_1^{n_1-1}\sin^{n_1-2}(\phi_{1,1})\cdots \sin(\phi_{1,n_1-2})}_{\delta = 1} 
\cdots 
\underbrace{r_k^{n_k-1}\sin^{n_k-2}(\phi_{k,1})\cdots \sin(\phi_{k,n_k-2})}_{\delta = k}\\[2ex]
&\hspace{0.8cm}\underbrace{dr_k d\phi_{k,1}\dots d\phi_{k, n_k - 1}}_{\delta = k}
\dots
\underbrace{dr_1 d\phi_{1,1}\dots d\phi_{1, n_1 - 1}}_{\delta = 1}
\end{align*}
\begin{align*}
&=
\underbrace{\int\displaylimits_{0}^{2\pi} \int\displaylimits_{0}^\pi \cdots \int\displaylimits_{0}^\pi \sin^{n_1-2}(\phi_{1,1})\cdots \sin(\phi_{1,n_1-2}) \;d\phi_{1,1}\dots d\phi_{1, n_1 - 1}}_{\delta = 1} \\[2ex]
&\hspace{1cm}\cdots\quad
\underbrace{\int\displaylimits_{0}^{2\pi} \int\displaylimits_{0}^\pi
\cdots \int\displaylimits_{0}^\pi \sin^{n_k-2}(\phi_{k,1})\cdots \sin(\phi_{k,n_k-2})\;d\phi_{k,1}\dots d\phi_{k, n_k - 1}}_{\delta = k}\\[2ex]
&\hspace{1cm}\int\displaylimits_{0}^r r_1^{n_1-1}\cdots\int\displaylimits_{0}^{r - \sum_{i=1}^{k-1} r_i} r_k^{n_k-1} \; dr_1\dots dr_k
\end{align*}
By applying Lemma \ref{lemma:AngleIntegral} and Lemma \ref{lemma:RadiusIntegral}, we can write this as:
\begin{align*}
V(r,\Delta) &= \left(2 \cdot \frac{\pi^{\frac{n_1}{2}}}{\Gamma(\frac{n_1}{2})}\right) \cdots \left(2 \cdot \frac{\pi^{\frac{n_k}{2}}}{\Gamma(\frac{n_k}{2})}\right) \cdot \frac{r^n}{\Gamma(n+1)} \prod_{i=1}^{k} \Gamma(n_i)\\
&= \frac{r^n}{\Gamma(n+1)} \cdot \prod_{i=1}^{k} \left(2 \cdot \pi^{\frac{n_i}{2}} \cdot \frac{\Gamma(n_i)}{\Gamma(\frac{n_i}{2})}\right)
\end{align*}
We can simplify this formula further by using the identity $\forall n \in \mathbb{N}: \Gamma(n+1) = n!$ and the rewrite $\prod_{i=0}^k \widehat{=} \prod_{\delta \in \Delta}$:

\begin{align*}
V(r,\Delta) &= \frac{r^n}{n!} \cdot \prod_{\delta \in \Delta} \left(2 \cdot \pi^{\frac{n_\delta}{2}} \cdot \frac{(n_\delta - 1)!}{\Gamma(\frac{n_\delta}{2})}\right) 
= \frac{r^n}{n!} \cdot \prod_{\delta \in \Delta} \left(\frac{2}{n_\delta} \cdot n_\delta! \cdot \frac{\pi^{\frac{n_\delta}{2}}}{\Gamma(\frac{n_\delta}{2})}\right)\\
&= \frac{r^n}{n!} \cdot \prod_{\delta \in \Delta} \left(n_\delta! \cdot \frac{\pi^{\frac{n_\delta}{2}}}{\frac{n_\delta}{2} \cdot \Gamma(\frac{n_\delta}{2})}\right)
= \frac{r^n}{n!} \cdot \prod_{\delta \in \Delta} \left(n_\delta! \cdot \frac{\pi^{\frac{n_\delta}{2}}}{\Gamma(\frac{n_\delta}{2}+1)}\right)
\end{align*}
The last transformation uses the fact that $\forall x \in \mathbb{R}^+: \Gamma(x) \cdot x = \Gamma(x+1)$. \qed
\end{proof}

\subsection{Hyperballs under the Weighted Metric}
\label{Size:Hyperellipses}

We now generalize our results from the previous section from the unweighted to the weighted combined metric $d_C^ \Delta$.

\begin{proposition}
The hypervolume of a hyperball with radius $r$ in a space with the weighted combined metric $d_C^\Delta$, the domain structure $\Delta$, and the set of weights $W$ can be computed by the following formula, where $n$ is the overall number of dimensions and $n_\delta$ is the number of dimensions in domain $\delta$:
$$V(r,\Delta, W) = \frac{1}{\prod_{\delta \in \Delta} w_{\delta} \cdot \prod_{d \in \delta} \sqrt{w_d}} \cdot \frac{r^n}{n!} \cdot \prod_{\delta \in \Delta} \left(n_\delta! \cdot \frac{\pi^{\frac{n_\delta}{2}}}{\Gamma(\frac{n_\delta}{2}+1)}\right)$$
\end{proposition}
\begin{proof}
As \cite[Chapter 1.6.4]{Gardenfors2000} has argued, putting weights on dimensions in a conceptual space is equivalent to stretching each dimension of the unweighted space by the weight assigned to it.

If the overall radius of a ball is $r$, and some dimension has the weight $w$, then the farthest away any point $x$ can be from the origin on this dimension must satisfy $w\cdot x = r$, i.e., $x = \frac{r}{w}$. That is, the ball needs to be stretched by a factor $\frac{1}{w}$ in the given dimension, thus its hypervolume also changes by a factor of $\frac{1}{w}$. A hyperball under the weighted metric is thus equivalent to a hyperellipse under the unweighted metric.

In our case, the weight for any dimension $d$ within a domain $\delta$ corresponds to $w_{\delta} \cdot \sqrt{w_{d}}$: If we look at a point $x$ with coordinates $(0, \dots, 0, x_d, 0, \dots, 0)$, then $d(0,x) = w_\delta \cdot \sqrt{w_d \cdot x_d^2} = w_\delta \cdot \sqrt{w_d} \cdot |x_d|$ (with $\delta$ being the domain to which the dimension $d$ belongs). If we multiply the size of the hyperball by $\frac{1}{w_\delta \cdot \sqrt{w_d}}$ for each dimension $d$, we get:
\begin{align*}
V(r,\Delta,W) &= \frac{1}{\prod_{\delta \in \Delta}\prod_{d \in \delta} w_{\delta} \sqrt{w_{d}}} \cdot V(r, \Delta)\\
&=\frac{1}{\prod_{\delta \in \Delta}\prod_{d \in \delta} w_{\delta} \sqrt{w_{d}}} \cdot \frac{r^n}{n!} \cdot \prod_{\delta \in \Delta} \left(n_\delta! \cdot \frac{\pi^{\frac{n_\delta}{2}}}{\Gamma(\frac{n_\delta}{2}+1)}\right)
\end{align*}
This is the hypervolume of a hyperball under the weighted combined metric. \qed
\end{proof}
%================================================================================================================================================================%

\section{Subsethood}
\label{Subsethood}

\setcounter{proposition}{2} % manually reset the proposition counter to be aligned with the paper
\begin{proposition}
\label{CF:proposition:CrispSubsethood}
Let $\widetilde{S}_1 = \langle S_1, \mu_0^{(1)}, c^{(1)}, W^{(1)}\rangle, \widetilde{S}_2 = \langle S_2, \mu_0^{(2)}, c^{(2)}, W^{(2)}\rangle$ be two concepts. Then, $\widetilde{S}_1 \subseteq \widetilde{S}_2$ if and only if the following conditions are true:
\begin{align*}
	\Delta_{S_2} \subseteq \Delta_{S_1} &\text{ and } \mu_0^{(1)} \leq \mu_0^{(2)} \text{ and } S_1 \subseteq {\widetilde{S}_2}^{\mu_0^{(1)}} \\
&\text{ and }\forall d \in D_{S_2}: c^{(1)} \cdot w^{(1)}_{\delta(d)} \cdot \sqrt{w^{(1)}_d} \geq c^{(2)} \cdot w^{(2)}_{\delta(d)} \cdot \sqrt{w^{(2)}_d}
\end{align*}
\end{proposition}
\begin{proof}
We show the two directions individually.\\

\underline{$\mathbf{\widetilde{S}_1 \subseteq \widetilde{S}_2 \Rightarrow} \textbf{ conditions}$}\\
Let $\widetilde{S}_1 \subseteq \widetilde{S}_2$, i.e., $\forall {x \in CS}: \mu_{\widetilde{S}_1}(x) \leq \mu_{\widetilde{S}_2}(x)$. We show the different parts of the implication individually:

\begin{itemize}

	\item $\boldsymbol{\mu_0}$: If $\mu_0^{(1)} > \mu_0^{(2)}$, then $\exists x \in S_1: \mu_{\widetilde{S}_1}(x) = \mu_0^{(1)} > \mu_0^{(2)} \geq \mu_{\widetilde{S}_2}(x)$, which is a contradiction to $\widetilde{S}_1 \subseteq \widetilde{S}_2$. Therefore, $\mu_0^{(1)} \leq \mu_0^{(2)}$.
	
	\item $\boldsymbol{S}$: If $S_1 \not\subseteq {\widetilde{S}_2}^{\mu_0^{(1)}}$, then $\exists x \in S_1: \mu_{\widetilde{S}_1}(x) = \mu_0^{(1)} > \mu_{\widetilde{S}_2}(x)$, which is a contradiction to $\widetilde{S}_1 \subseteq \widetilde{S}_2$. Therefore, $S_1 \subseteq {\widetilde{S}_2}^{\mu_0^{(1)}}$.
	
	\item $\boldsymbol{\Delta}$: If $\Delta_{S_2} \not\subseteq \Delta_{S_1}$, then $\exists \delta \in \Delta_{S_2}: \delta \notin \Delta_{S_1}$. Pick any $d \in \delta$ and any $x \in CS$. Even if $\mu_{\widetilde{S}_1}(x) \leq \mu_{\widetilde{S}_1}(x)$, we can modify $x_d$ in such a way that $x_d \gg p_d^+$ for all $C \in S_2$. Thus, even for small $w_\delta, w_d$, we can arbitrarily increase $d_C^\Delta(x, S_2, W^{(2)})$ by choosing $x_d$ large enough. As $\delta \notin \Delta_{S_1}$, the modification of $x_d$ does not change $d_C^{\Delta}(x, S_1, W^{(1)})$. Therefore, at some point we get $d_C^\Delta(x, S_2, W^{(2)}) \gg d_C^\Delta(x, S_1, W^{(1)})$ and thus $\mu_{\widetilde{S}_1}(x) > \mu_{\widetilde{S}_2}(x)$, which is a contradiction to $\widetilde{S}_1 \subseteq \widetilde{S}_2$. Therefore, $\Delta_{S_2} \subseteq \Delta_{S_1}$.
	
	\item $\boldsymbol{W}$: We know that for all $x \in CS$: $$\mu_{\widetilde{S}_1}(x) = \mu_0^{(1)} \cdot e^{-c^{(1)}\cdot d_C^{\Delta_{S_1}}(x, S_1, W^{(1)})} \leq \mu_0^{(2)} \cdot e^{-c^{(2)}\cdot d_C^{\Delta_{S_2}}(x, S_2, W^{(2)})} = \mu_{\widetilde{S}_2}(x)$$ Assume for some $d^* \in D_{S_2}$ that $c^{(1)} \cdot w^{(1)}_{\delta(d^*)} \cdot \sqrt{w^{(1)}_{d^*}} < c^{(2)} \cdot w^{(2)}_{\delta(d^*)} \cdot \sqrt{w^{(2)}_{d^*}}$. Assume for now also that $S_1 = S_2$. Pick $y^{(1)} = y^{(2)} \in S_1 = S_2$ on the upper border of $S_1$ with respect to $d^*$ and define $\forall d \in D \setminus \{d^*\}: x_d := y_d$ and $x_{d^*} > y_{d^*}$.
	\begin{align*}
		\mu_{\widetilde{S}_2}(x) &= \mu_0^{(2)} \cdot e^{-c^{(2)}\cdot \min_{y \in S_2} \sum_{\delta \in \Delta_2} w_\delta^{(2)} \cdot \sqrt{\sum_{d \in \delta} w_d^{(2)} \cdot |x_d - y_d|^2}}\\
		&= \mu_0^{(2)} \cdot e^{-\min_{y \in S_2} \sum_{\delta \in \Delta_2} \sqrt{\sum_{d \in \delta} \left(c^{(2)} \cdot w_\delta^{(2)} \cdot \sqrt{w_d^{(2)}}\right)^2 \cdot \left|x_d - y_d\right|^2}}\\
		&\stackrel{(i)}{=} \mu_0^{(2)} \cdot e^{- \left( c^{(2)} \cdot w_{\delta(d^*)}^{(2)} \cdot \sqrt{w_{d^*}^{(2)}} \right) \cdot \left|x_{d^*} - y_{d^*}\right|}\\
		&\stackrel{(ii)}{=} \mu_0^{(2)} \cdot e^{- \left( c^{(1)} \cdot w_{\delta(d^*)}^{(1)} \cdot \sqrt{w_{d^*}^{(1)}} + \epsilon \right) \cdot \left|x_{d^*} - y_{d^*}\right|}\\
		&= \mu_0^{(2)} \cdot e^{- \left(c^{(1)} \cdot w_{\delta(d^*)}^{(1)} \cdot \sqrt{w_{d^*}^{(1)}}\right) \cdot \left|x_{d^*} - y_{d^*}\right|} \cdot e^{- \epsilon \cdot \left|x_{d^*} - y_{d^*}\right|}
	\end{align*}
	In step $(i)$ we used that $\forall d \in D \setminus \{d^*\}: x_d = y_d$ and in step $(ii)$ we used that $c^{(1)} \cdot w^{(1)}_{\delta(d^*)} \cdot \sqrt{w^{(1)}_{d^*}} < c^{(2)} \cdot w^{(2)}_{\delta(d^*)} \cdot \sqrt{w^{(2)}_{d^*}}$.
	\begin{align*}
	&\mu_{\widetilde{S}_2}(x) < \mu_0^{(1)} \cdot e^{- \left(c^{(1)} \cdot w_{\delta(d^*)}^{(1)} \cdot \sqrt{w_{d^*}^{(1)}}\right) \cdot \left|x_{d^*} - y_{d^*}\right|} = \mu_{\widetilde{S}_1}(x)\\
	&\Leftrightarrow \mu_0^{(2)} \cdot e^{- \epsilon \cdot \left|x_{d^*} - y_{d^*}\right|} < \mu_0^{(1)} \Leftrightarrow \left|x_{d^*} - y_{d^*}\right| > -\frac{1}{\epsilon} \ln \left(\frac{\mu_0^{(1)}}{\mu_0^{(2)}}\right)
	\end{align*}
	If we pick $x_{d^*}$ large enough, then $\mu_{\widetilde{S}_1}(x) > \mu_{\widetilde{S}_2}(x)$, which is a contradiction to $\widetilde{S}_1 \subseteq \widetilde{S}_2$.\\
	
	Let us now remove the simplifying assumption of $S_1 = S_2$. If we construct $x$ based on $y^{(1)} \in S_1$ then we still know that $\forall d \in D \setminus \{d^*\}: \left|x_{d} - y_{d}^{(1)}\right| = 0$. However, for the optimal $y^{(2)} \in S_2$, we only know that $\left|x_{d} - y_{d}^{(2)}\right| \geq 0 = \left|x_{d} - y_{d}^{(1)}\right|$. This means that the distance from $x$ to the closest $y^{(2)} \in S_2$ with respect to all dimensions $d \in D \setminus \{d^*\}$ is as least as large as its distance to the closest $y^{(1)} \in S_1$. However, we can no longer assume that $\left|x_{d^*} - y_{d^*}^{(1)}\right| = \left|x_{d^*} - y_{d^*}^{(2)}\right|$. 
	We can still ensure that
	\begin{align*}c^{(1)} w_{\delta(d^*)}^{(1)} \sqrt{w_{d^*}^{(1)}} \left|x_{d^*} - y_{d^*}^{(1)}\right| &< c^{(2)} w_{\delta(d^*)}^{(2)} \sqrt{w_{d^*}^{(2)}} \left|x_{d^*} - y_{d^*}^{(2)}\right|\\
	 &\quad= \left(c^{(1)} w_{\delta(d^*)}^{(1)} \sqrt{w_{d^*}^{(1)}} + \epsilon\right) \left|x_{d^*} - y_{d^*}^{(2)}\right|
	\end{align*} if and only if 
	$$\epsilon \left(x_{d^*} - y_{d^*}^{(2)} \right) > c^{(1)} w_{\delta(d^*)}^{(1)} \sqrt{w_{d^*}^{(1)}} \underbrace{\left(\left|x_{d^*} - y_{d^*}^{(1)}\right| - \left|x_{d^*} - y_{d^*}^{(2)}\right| \right)}_{=:a}	$$
	In the formula above, $a$ is a constant indepentent of $x_{d^*}$ and $c^{(1)} w_{\delta(d^*)}^{(1)} \sqrt{w_{d^*}^{(1)}}$ is positive and constant as well. Also $\epsilon$ is guaranteed to be positive. If $a$ is negative, then the right hand side of the formula becomes negative as well and we can pick any $x_{d^*} \geq y_{d^*}^{(2)}$ in order to make the formula true. If $a$ is positive, we can still choose $x_{d^*} \gg y_{d^*}^{(2)}$ large enough to make the above inequality true. 
	
	So if $c^{(1)} \cdot w^{(1)}_{\delta(d^*)} \cdot \sqrt{w^{(1)}_{d^*}} < c^{(2)} \cdot w^{(2)}_{\delta(d^*)} \cdot \sqrt{w^{(2)}_{d^*}}$ for any $d^* \in D_{S_2}$, then we can find an $x \in CS$ with $\mu_{\widetilde{S}_2}(x) < \mu_{\widetilde{S}_1}(x)$, thus $\widetilde{S}_1 \not\subseteq \widetilde{S}_2$. Therefore, $\forall d \in D_{S_2}: c^{(1)} \cdot w^{(1)}_{\delta(d)} \cdot \sqrt{w^{(1)}_d} \geq c^{(2)} \cdot w^{(2)}_{\delta(d)} \cdot \sqrt{w^{(2)}_d}$.
\end{itemize}

\underline{$ \textbf{conditions } \mathbf{\Rightarrow \widetilde{S}_1 \subseteq \widetilde{S}_2}$}\\
Let $\widetilde{S}_1, \widetilde{S}_2$ with the properties as proposed, and let $x \in CS$ be arbitrary but fixed.
\begin{align*}
	\mu_{\widetilde{S}_1}(x) &= \mu_0^{(1)} \cdot e^{-c^{(1)}\cdot \min_{y \in S_1} \sum_{\delta \in \Delta_1} w_\delta^{(1)} \cdot \sqrt{\sum_{d \in \delta} w_d^{(1)} \cdot |x_d - y_d|^2}}\\
	&\stackrel{(i)}{\leq} \mu_0^{(1)} \cdot e^{-c^{(1)}\cdot \min_{y \in \widetilde{S}_2^{\mu_0^{(1)}}} \sum_{\delta \in \Delta_1} w_\delta^{(1)} \cdot \sqrt{\sum_{d \in \delta} w_d^{(1)} \cdot |x_d - y_d|^2}}\\
	&\stackrel{(ii)}{\leq} \mu_0^{(1)} \cdot e^{-c^{(1)}\cdot \min_{y \in \widetilde{S}_2^{\mu_0^{(1)}}} \sum_{\delta \in \Delta_2} w_\delta^{(1)} \cdot \sqrt{\sum_{d \in \delta} w_d^{(1)} \cdot |x_d - y_d|^2}}\\
	&= \mu_0^{(1)} \cdot e^{- \min_{y \in \widetilde{S}_2^{\mu_0^{(1)}}} \sum_{\delta \in \Delta_2} \sqrt{\sum_{d \in \delta} \left(c^{(1)}\cdot w_\delta^{(1)} \cdot \sqrt{w_d^{(1)}}\right)^2 \cdot |x_d - y_d|^2}}\\
	&\stackrel{(iii)}{\leq} \mu_0^{(1)} \cdot e^{- \min_{y \in \widetilde{S}_2^{\mu_0^{(1)}}} \sum_{\delta \in \Delta_2} \sqrt{\sum_{d \in \delta} \left(c^{(2)}\cdot w_\delta^{(2)} \cdot \sqrt{w_d^{(2)}}\right)^2 \cdot |x_d - y_d|^2}}\\
	&= \mu_0^{(1)} \cdot e^{-c^{(2)} \cdot \min_{y \in \widetilde{S}_2^{\mu_0^{(1)}}} \sum_{\delta \in \Delta_2} \cdot w_\delta^{(2)} \cdot \sqrt{\sum_{d \in \delta} w_d^{(2)} \cdot |x_d - y_d|^2}} =: f(x)
\end{align*}
In step (i) we use that $S_1 \subseteq \widetilde{S}_2^{\mu_0^{(1)}}$ and in step (ii) that $\Delta_{S_2} \subseteq \Delta_{S_1}$. In step (iii) we use that $\forall d \in D_{S_2}: c^{(1)} \cdot w^{(1)}_{\delta(d)} \cdot \sqrt{w^{(1)}_d} \geq c^{(2)} \cdot w^{(2)}_{\delta(d)} \cdot \sqrt{w^{(2)}_d}$.

For all $x \in \widetilde{S}_2^{\mu_0^{(1)}}$ we get that $f(x) = \mu_0^{(1)} \leq \mu_{\widetilde{S}_2}(x)$. Let now $x \notin \widetilde{S}_2^{\mu_0^{(1)}}$. It is easy to see that for the point $y \in S_2$ with the smallest distance to $x$, there is an $z \in CS$ with $B(x,z,y)$ and $\mu_{\widetilde{S}_2}(z) = \mu_0^{(1)}$. Then:
\begin{align*}
 \mu_{\widetilde{S}_2}(x) &= \mu_0^{(2)} \cdot e^{-c^{(2)} \cdot d_C^{\Delta_{S_2}}(x,y,W^{(2)})} = \mu_0^{(2)} \cdot e^{-c^{(2)} \cdot \left(d_C^{\Delta_{S_2}}(x,z,W^{(2)}) + d_C^{\Delta_{S_2}}(z,y,W^{(2)})\right)}\\
 &= \underbrace{\mu_0^{(2)} \cdot e^{-c^{(2)} \cdot d_C^{\Delta_{S_2}}(z,y,W^{(2)})}}_{= \mu_0^{(1)}}\; \cdot\; e^{- c^{(2)} \cdot d_C^{\Delta_{S_2}}(x,z,W^{(2)})}
\end{align*}
Instead of calculating $\mu_{\widetilde{S}_2}(x)$ based on the distance between $x$ and $S_2$, we can calculate $\mu_{\widetilde{S}_2}(x)$ based on the distance between $x$ and $\widetilde{S}_2^{\mu_0^{(1)}}$ and rescale the result such that it is at most $\mu_0^{(1)}$. Therefore, for $x \notin \widetilde{S}_2^{\mu_0^{(1)}}$, we get that $f(x) = \mu_{\widetilde{S}_2}(x)$. Thus, $\mu_{\widetilde{S}_1}(x) \leq f(x) \leq \mu_{\widetilde{S}_2}(x)$ for all $x \in CS$ and therefore $\widetilde{S}_1 \subseteq \widetilde{S}_2$.\qed
\end{proof}

%================================================================================================================================================================%
\section{Similarity}
\label{Similarity}

\setcounter{definition}{5}
\begin{definition}
\label{def:Similarity}
A function $Sim(\widetilde{S}_1, \widetilde{S}_2) \in [0,1]$ is called a similarity function, if it has the following properties for all concepts $\widetilde{S}_1, \widetilde{S}_2$:
\begin{enumerate}
	\item $Sim(\widetilde{S}_1, \widetilde{S}_2) = 1.0 \Rightarrow Sub(\widetilde{S}_1, \widetilde{S}_2) = 1.0$
	\item $\widetilde{S}_1 = \widetilde{S}_2 \Rightarrow Sim(\widetilde{S}_1, \widetilde{S}_2) = 1.0$
	\item $\widetilde{S}_1 \subseteq \widetilde{S}_2 \Rightarrow Sim(\widetilde{S}_1, \widetilde{S}_2) \geq Sim(\widetilde{S}_2, \widetilde{S}_1)$
\end{enumerate}
\end{definition}

\begin{proposition}
$Sim_S(\widetilde{S}_1, \widetilde{S}_2) := Sub(\widetilde{S}_1, \widetilde{S}_2)$ is a similarity function according to Definition \ref{def:Similarity}.
\end{proposition}
\begin{proof}
We show the different properties individually.
\begin{enumerate}
	\item Trivial.
	\item Let $\widetilde{S}_1 = \widetilde{S}_2$. Then also $\widetilde{S}_1 \subseteq \widetilde{S}_2$ and thus $Sim_S(\widetilde{S}_1, \widetilde{S}_2) = Sub(\widetilde{S}_1, \widetilde{S}_2) = 1.0$
	\item $\widetilde{S}_1 \subseteq \widetilde{S}_2$ implies that $Sub(\widetilde{S}_1, \widetilde{S}_2) = 1.0$ which means by definition that $Sim_S(\widetilde{S}_1, \widetilde{S}_2) = 1.0$. As $Sim_S(\widetilde{S}_2, \widetilde{S}_1) \in [0,1]$, we get that $Sim_S(\widetilde{S}_1, \widetilde{S}_2) \geq Sim_S(\widetilde{S}_2, \widetilde{S}_1)$. \qed
\end{enumerate}
\end{proof}

\begin{proposition}
$Sim_J(\widetilde{S}_1, \widetilde{S}_2) := \frac{M(\widetilde{S}_1 \cap \widetilde{S}_2)}{M(\widetilde{S}_1 \cup \widetilde{S}_2)}$ is a similarity function according to Definition \ref{def:Similarity}.
\end{proposition}
\begin{proof}
We show the different properties individually.
\begin{enumerate}
	\item Let $\widetilde{S}_I := I(\widetilde{S}_1, \widetilde{S}_2)$ be the intersection of $\widetilde{S}_1$ and $\widetilde{S}_2$ as defined in \cite{Bechberger2017KI} and $\widetilde{S}_U := U(\widetilde{S}_1, \widetilde{S}_2)$ their unification. As we know from the definitions of intersection and unification \cite{Bechberger2017KI}, $c^{(I)} = c^{(U)}$ and $W^{(I)} = W^{(U)}$. It is easly to see that $\mu_0^{(I)} \leq \mu_0^{(U)}$ and that $|S_I| \leq |S_U|$. Therefore, $Sim_J(\widetilde{S}_1, \widetilde{S}_2) = 1.0$ can only happen if $\widetilde{S}_I = \widetilde{S}_U$ (which implies that also $\mu_0^{(I)} = \mu_0^{(U)}$). The observation that $\mu_0^{(I)} = \mu_0^{(U)}$ implies that $\mu_0^{(1)} = \mu_0^{(2)}$ (as $\mu_0^{(I)} \leq \min(\mu_0^{(1)}, \mu_0^{(2)})$ and $\mu_0^{(U)} = \max(\mu_0^{(1)}, \mu_0^{(2)})$) and that $S_1 \cap S_2 \neq \emptyset$. If $\widetilde{S}_I = \widetilde{S}_U$, then also $S_I = S_U$. Then, $I(S_1, S_2) = S_I$ (because $S_1 \cap S_2 \neq \emptyset$ and $\mu_0^{(1)} = \mu_0^{(2)}$) and $U(S_1, S_2) = S_U$ (by definition of the unification). Thus, $I(S_1, S_2) = U(S_1, S_2)$. This is only possible if $S_1 = S_2$. Then, $Sub(\widetilde{S}_1, \widetilde{S}_2) = \frac{M(I(\widetilde{S}_1, \widetilde{S}_2))_2}{M(\widetilde{S}_1)_2} = \frac{M(\widetilde{S}_1)_2}{M(\widetilde{S}_1)_2} = 1.0$ (where we use $M(\widetilde{S})_i$ to denote that $W^{(i)}$ and $c^{(i)}$ are used for computing the measure, cf. Section 3.2). 
	\item Let $\widetilde{S}_1 = \widetilde{S}_2$. Obviously, $I(\widetilde{S}_1, \widetilde{S}_2) = \widetilde{S}_1 = U(\widetilde{S}_1, \widetilde{S}_2)$. We therefore get $Sim_J(\widetilde{S}_1, \widetilde{S}_2) = \frac{M(I(\widetilde{S}_1, \widetilde{S}_2))}{M(U(\widetilde{S}_1, \widetilde{S}_2))} = 1.0$.
	\item Because $I(\widetilde{S}_1, \widetilde{S}_2) = I(\widetilde{S}_2, \widetilde{S}_1)$ and $U(\widetilde{S}_1, \widetilde{S}_2) = U(\widetilde{S}_2, \widetilde{S}_1)$, we get that $Sim_J(\widetilde{S}_1, \widetilde{S}_2) = Sim_J(\widetilde{S}_2, \widetilde{S}_1)$ for all $\widetilde{S}_1, \widetilde{S}_2$. Thus the consequent of the implication is always true. \qed
\end{enumerate}
\end{proof}

\end{document}